%% file: main.tex
\pdfoutput=1

\documentclass[runningheads]{llncs}
\usepackage{graphicx}

\usepackage{tikz}
\usepackage{comment}
\usepackage{amsmath,amssymb} 
\usepackage{color}

\usepackage[accsupp]{axessibility}  

\usepackage{hyperref}       
\hypersetup{colorlinks=false,linkcolor=blue}
\usepackage{url}            
\usepackage{booktabs}       
\usepackage{amsfonts}       
\usepackage{nicefrac}       
\usepackage{microtype}      
\usepackage{xcolor}         
\usepackage{bm}
\usepackage{float}
\usepackage{makecell}
\usepackage{adjustbox}
\usepackage{multirow}
\usepackage{wrapfig}
\usepackage{subcaption}
\usepackage{cite}
\usepackage{bbm}
\usepackage{colortbl}
\usepackage{tabularx}
\usepackage[normalem]{ulem}

\newtheorem{defn}{Definition}

\begin{document}
\pagestyle{headings}
\mainmatter
\def\ECCVSubNumber{4633}  

\title{Equivariant Hypergraph Neural Networks} 

\titlerunning{Equivariant Hypergraph Neural Networks}
%
\author{Jinwoo Kim\inst{1} \and 
Saeyoon Oh\inst{1} \and
Sungjun Cho\inst{2} \and
Seunghoon Hong\inst{1,2}}
\authorrunning{J. Kim et al.}
%
\institute{
$^1$KAIST\\
$^2$LG AI Research\\
}
\maketitle

\begin{abstract}
Many problems in computer vision and machine learning can be cast as learning on hypergraphs that represent higher-order relations.
Recent approaches for hypergraph learning extend graph neural networks based on message passing, which is simple yet fundamentally limited in modeling long-range dependencies and expressive power.
On the other hand, tensor-based equivariant neural networks enjoy maximal expressiveness, but their application has been limited in hypergraphs due to heavy computation and strict assumptions on fixed-order hyperedges.
We resolve these problems and present Equivariant Hypergraph Neural Network (EHNN), the first attempt to realize maximally expressive equivariant layers for general hypergraph learning.
We also present two practical realizations of our framework based on hypernetworks (EHNN-MLP) and self-attention (EHNN-Transformer), which are easy to implement and theoretically more expressive than most message passing approaches.
We demonstrate their capability in a range of hypergraph learning problems, including synthetic $k$-edge identification, semi-supervised classification, and visual keypoint matching, and report improved performances over strong message passing baselines.
Our implementation is available at \url{https://github.com/jw9730/ehnn}.
\end{abstract}

\input{introduction}
\input{preliminary}
\input{methods}
\input{experiments}
\input{conclusion}

\paragraph{Acknowledgement}
This work was supported by  Institute of Information \& communications
Technology Planning \& Evaluation (IITP) (No. 2021-0-00537, 2019-0-00075 and 2021-0-02068), the National
Research Foundation of Korea (NRF) (No. 2021R1C1C1012540 and 2021R1A4A3032834), and Korea Meteorological Administration Research and Development Program "'Development of AI techniques for Weather Forecasting" under Grant (KMA2021-00121).

\clearpage
%
%
\bibliographystyle{splncs04}
\bibliography{main}

\newpage
\input{appendix}


\end{document}

%% file: introduction.tex
\section{Introduction}\label{sec:introduction}
Reasoning about a system that involves a set of entities and their relationships requires relational data structures.
Graph represents relational data with nodes and edges, where a node corresponds to an entity and an edge represents a relationship between a pair of nodes.
However, pairwise edges are often insufficient to represent more complex relationships.
For instance, many geometric configurations of entities such as angles and areas can only be captured by considering higher-order relationships between three or more nodes.
Hypergraph is a general data structure that represents such higher-order relationships with hyperedges, \emph{i.e.}, edges associating more than two nodes at a time~\cite{berge1981graphs}.
Thus, it is widely used to represent various visual data such as scenes~\cite{gao20123D, kim2020hypergraph}, feature correspondence~\cite{wang2019neural, belongie2002shape, ray20052D, lowe1999object}, and polygonal mesh~\cite{botsch2008geometric, milano2020primal, yavartanoo2021polynet}, as well as general relational data such as social networks~\cite{tan2011using, li2013link, bu2010music}, biological networks~\cite{gu2017functional, klimm2021hypergraphs}, linguistic structures~\cite{ding2020be}, and combinatorial optimization problems~\cite{kalai1997linear}.

To learn deep representation of hypergraphs, recent works developed specialized hypergraph neural networks by generalizing the message passing operator of graph neural networks (GNNs)~\cite{feng2019hypergraph, bai2021hypergraph, dong2020hnhn, huang2021unignn, arya2020hypersage, chien2022you}.
In these networks, node and (hyper)edge features are updated recurrently by aggregating features of neighboring nodes and edges according to the connectivity of the input (hyper)graph.
Despite such simplicity, message-passing networks have fundamental limitations.
Notably, the local and recurrent characteristics of message passing prevent them from handling dependencies between any pair of nodes with a distance longer than the number of propagation steps~\cite{gu2020implicit, kim2021transformers}.
It is also known that this locality is related to oversmoothing, which hinders the use of deep networks~\cite{li2018deeper, cai2020a, oono2020graph, huang2021unignn}.

A more general and potentially powerful approach for hypergraph learning is to find \emph{all} possible permutation equivariant linear operations on the input (hyper)graph and use them as bases of linear layers that constitute \emph{equivariant GNNs}~\cite{maron2019invariant, velikovic2022message}.
While message passing is one specific, locally restricted case of equivariant operation, the maximal set of equivariant operations extends further, involving various global interactions over possibly disconnected nodes and (hyper)edges~\cite{kim2021transformers}.
The formulation naturally extends to higher-order layers that can handle hypergraphs in principle and even mixed-order layers where input and output are of different orders (e.g., graph in, hypergraph out).
Despite the advantages, the actual usage of equivariant GNNs has been mainly limited to sets and graphs~\cite{zaheer2017deep, maron2019invariant, serviansky2020set, kim2021transformers}, and they have not been realized for general hypergraph learning.
This is mainly due to the prohibitive parameter dimensionality of higher-order layers and a bound in the input and output hyperedge orders that comes from fixed-order tensor representation.

We propose \emph{Equivariant Hypergraph Neural Network (EHNN)} as the first attempt to realize equivariant GNNs for general hypergraph learning.
We begin by establishing a simple connection between sparse, arbitrarily structured hypergraphs and dense, fixed-order tensors, from which we derive the maximally expressive equivariant linear layer for undirected hypergraphs.
Then, we impose an intrinsic parameter sharing within the layer via hypernetworks~\cite{ha2017hypernetworks}, which (1) retains maximal expressiveness, (2) practically bounds the number of parameters, and (3) allows processing hyperedges with arbitrary and possibly unseen orders.
Notably, the resulting layer (EHNN-MLP) turns out to be a simple augmentation of an MLP-based message passing with hyperedge order embedding and global pooling.
This leads to efficient implementation and allows the incorporation of any advances in the message-passing literature.
We further extend into a Transformer counterpart (EHNN-Transformer) by introducing self-attention to achieve a higher expressive power with the same asymptotic cost.
In a challenging synthetic $k$-edge identification task where message passing networks fail, we show that the high expressiveness of EHNN allows fine-grained global reasoning to perfectly solve the task, and demonstrate their generalizability towards unseen hyperedge orders.
We also demonstrate their state-of-the-art performance in several transductive and inductive hypergraph learning benchmarks, including semi-supervised classification and visual correspondence matching.

%% file: preliminary.tex
\section{Preliminary and Related Work}\label{sec:preliminary}
Let us introduce some preliminary concepts of permutation equivariant learning~\cite{maron2019invariant, serviansky2020set, kim2021transformers}.
We first describe higher-order tensors and then maximally expressive permutation equivariant linear layers that compose equivariant GNNs~\cite{maron2019invariant}.

We begin with some notations.
We denote a set as $\{a, ..., b\}$, a tuple as $(a, ..., b)$, and denote $[n]=\{1, ..., n\}$.
We denote the space of order-$k$ tensors as $\mathbb{R}^{n^k \times d}$ with feature dimensionality $d$.
For an order-$k$ tensor $\mathbf{A}\in\mathbb{R}^{n^k\times d}$, we use a multi-index $\mathbf{i}=(i_1, ..., i_k)\in[n]^k$ to index an element $\mathbf{A}_{\mathbf{i}} = \mathbf{A}_{i_1, ..., i_k}\in\mathbb{R}^d$.
Let $S_n$ denote all permutations of $[n]$.
A node permutation $\pi \in S_n$ acts on a multi-index~$\mathbf{i}$ by $\pi(\mathbf{i}) = (\pi(i_1), ..., \pi(i_k))$ and acts on a tensor $\mathbf{A}$ by $(\pi\cdot\mathbf{A})_{\mathbf{i}}=\mathbf{A}_{\pi^{-1}(\mathbf{i})}$.

\subsubsection{Higher-order Tensors}
Prior work on equivariant learning regard hypergraph data as $G=(V, \mathbf{A})$, with $V$ a set of $n$ nodes and $\mathbf{A}\in\mathbb{R}^{n^k \times d}$ a tensor encoding hyperedge features~\cite{maron2019invariant}.
The order $k$ of the tensor $\mathbf{A}$ indicates the type of hypergraph.
First-order tensor encodes a set of features (\emph{e.g.}, point cloud) where $\mathbf{A}_i$ is the feature of node $i$.
Second-order tensor encodes pairwise edge features (\emph{e.g.}, adjacency) where $\mathbf{A}_{i_1, i_2}$ is the feature of edge $(i_1, i_2)$.
Generally, an order-$k$ tensor encodes \emph{hyperedge} features (\emph{e.g.}, mesh normal) where $\mathbf{A}_{i_1, ..., i_k}$ is the feature of hyperedge $(i_1, ..., i_k)$.
We begin our discussion from the tensors, but will arrive at the familiar notion of hypergraphs with any-order undirected hyperedges.

\subsubsection{Permutation Invariance and Equivariance}
In (hyper)graph learning, we are interested in building a function $f$ that takes a (higher-order) tensor $\mathbf{A}$ as input and outputs some value.
Since the tensor representation of a graph changes dramatically with the permutation of node numbering, the function $f$ should be invariant or equivariant under node permutations.
Formally, if the output is a single vector, $f$ is required to be \emph{permutation invariant}, always satisfying $f(\pi\cdot\mathbf{A}) = f(\mathbf{A})$; if the output is a tensor, $f$ is required to be \emph{permutation equivariant}, always satisfying $f(\pi\cdot\mathbf{A})=\pi\cdot f(\mathbf{A})$~\cite{maron2019invariant}.
As a neural network $f$ is often built as a stack of linear layers and non-linearities, its construction reduces to finding invariant and equivariant \emph{linear} layers.

\subsubsection{Invariant and Equivariant Linear Layers}
Many (hyper)graph neural networks rely on message passing~\cite{gilmer2017neural, feng2019hypergraph}, which is a restricted equivariant operator.
Alternatively, \emph{maximally expressive} linear layers for higher-order tensors have been characterized by Maron~el~al.~(2019)~\cite{maron2019invariant}.
Specifically, invariant linear layers $L_{k\to 0}: \mathbb{R}^{n^k\times d}\to\mathbb{R}^{d'}$ and equivariant linear layers $L_{k\to l}: \mathbb{R}^{n^k\times d}\to\mathbb{R}^{n^l\times d'}$ were identified (note that invariance is a special case of equivariance with $l=0$).
Given an order-$k$ input $\mathbf{A}\in\mathbb{R}^{n^k\times d}$, the order-$l$ output of an equivariant linear layer $L_{k\rightarrow l}$ is written as follows, with indicator $\mathbbm{1}$ and multi-indices $\mathbf{i}\in[n]^k,\mathbf{j}\in[n]^l$:
\begin{align}\label{eqn:equivariant_linear_layer}
    L_{k\rightarrow l}(\mathbf{A})_{\mathbf{j}} = \sum_{\mu}{\sum_{\mathbf{i}}{\mathbbm{1}_{(\mathbf{i}, \mathbf{j})\in\mu}\mathbf{A}_{\mathbf{i}}w_{\mu}}} + \sum_{\lambda}{\mathbbm{1}_{\mathbf{j}\in\lambda}b_{\lambda}},
\end{align}
where $w_\mu\in\mathbb{R}^{d\times d'}$, $b_\lambda\in\mathbb{R}^{d'}$ are weight and bias parameters, and $\mu$ and $\lambda$ are \emph{equivalence classes} of order-$(k+l)$ and order-$l$ multi-indices, respectively.

The equivalence classes can be interpreted as a partitioning of a multi-index space.
The order-$(k+l)$ equivalence classes $\mu$ for the weight specifies a partitioning of the space of multi-indices $[n]^{k+l}$, and order-$l$ equivalence classes $\lambda$ for the bias specifies a partitioning of the space of multi-indices $[n]^l$.
The total number of the equivalence classes (the size of partitioning) depends only on orders $k$ and $l$.
With $\text{b}(k)$ the $k$-th Bell number, there exist $\text{b}(k+l)$ equivalence classes $\mu$ for the weight and $\text{b}(l)$ equivalence classes $\lambda$ for the bias.
For the first-order layer $L_{1\to 1}$, there exist $\text{b}(2)=2$ equivalence classes $\mu_1$, $\mu_2$ for the weight, specifying the partitioning of $[n]^2$ as $\{\mu_1, \mu_2\}$ where $\mu_1 = \{(i, j)|i = j\}$ and $\mu_2 = \{(i, j)|i \neq j\}$.
For further details including derivations and illustrative descriptions, we guide the readers to Maron~et~al.~(2019)~\cite{maron2019invariant} and Kim~et~al.~(2021)~\cite{kim2021transformers}.

\subsubsection{Equivariant GNNs}
Based on the maximally expressive equivariant linear layers (Eq.~\eqref{eqn:equivariant_linear_layer}), a bouquet of permutation invariant or equivariant neural networks was formulated.
A representative example is \emph{equivariant GNN}~\cite{maron2019invariant} (also called $k$-IGN~\cite{maron2019provably, chen2020can}), built by stacking the equivariant linear layers and non-linearity.
Their theoretical expressive power has been extensively studied~\cite{zaheer2017deep, keriven2019universal, maron2019provably, maron2019on, chen2020can}, leading to successful variants in set and graph learning~\cite{maron2020on, serviansky2020set, kim2021transformers, kim2022pure}.
In particular, practical variants such as Higher-order Transformer~\cite{kim2021transformers} and TokenGT~\cite{kim2022pure} unified the equivariant GNNs and the Transformer architecture~\cite{vaswani2017attention, lee2019set}, surpassing the performance of message-passing GNNs in large-scale molecular graph regression.

\subsubsection{Challenges in Hypergraph Learning}
Despite the theoretical and practical advantages, to our knowledge, equivariant GNN and its variants were rarely considered for general hypergraph learning involving higher-order data~\cite{albooyeh2019incidence}, and never implemented except for highly restricted $k$-uniform hyperedge prediction~\cite{serviansky2020set, kim2021transformers}.
We identify two main challenges.
First, although the asymptotic cost can be reduced to a practical level with recent tricks~\cite{kim2021transformers}, the number of parameters still grows rapidly to Bell number of input order~\cite{berend2010improved}.
This makes any layer $L_{k\to l}$ with $k + l > 4$ challenging to use, as $k + l = 5$ already leads to $52$ weight matrices.
Second, in inductive learning~\cite{william2017inductive, zhang2018an} where a model is tested on unseen nodes or hypergraphs, the model can be required to process unseen-order hyperedges that possibly surpass the max order in the training data.
This is not straightforward for equivariant GNNs, because the fixed-order tensors that underlie $L_{k\to l}$ require pre-specifying the max hyperedge order $(k, l)$ that the model can process.

%% file: methods.tex
\section{Equivariant Hypergraph Neural Network}\label{sec:methods}
We now proceed to our framework on practical equivariant GNNs for general hypergraph data.
All proofs can be found in Appendix~\ref{sec:apdx_proofs}.
In practical setups that assume undirected hypergraphs~\cite{feng2019hypergraph, yadati2019hypergcn, arya2020hypersage, dong2020hnhn, bai2021hypergraph, chien2022you}, a hypergraph $G = (V, E, \mathbf{X})$ is defined by a set of $n$ nodes $V$, a set of $m$ hyperedges $E$, and features $\mathbf{X}\in\mathbb{R}^{m\times d}$ of the hyperedges.
Each hyperedge $e\in E$ is a subset of node set $V$, and its order $|e|$ indicates its type.
For example, a first-order edge $\{i\}$ represents an $i$-th node; a second-order edge $\{i, j\}$ represents a pairwise link of $i$-th and $j$-th nodes; in general, an order-$k$ edge $\{i_1, ..., i_k\}$ represents a hyperedge that links $k$ nodes.
By $\mathbf{X}_e\in\mathbb{R}^d$ we denote the feature attached to a hyperedge $e$.
We assume that node and hyperedge features are both $d$-dimensional~\cite{maron2019invariant, maron2019provably, kim2021transformers, chen2020can}; to handle different dimensionalities, we simply let $d=(d_v+d_e)$ and place node features at first $d_v$ channels and hyperedge features at last $d_e$ channels.

Note that the above notion of hypergraphs $(V, E, \mathbf{X})$ does not directly align with the higher-order tensors $\mathbf{A}\in\mathbb{R}^{n^k\times d}$ described in Section~\ref{sec:preliminary}.
Unlike them, hypergraphs of our interest are sparse, their hyperedges are undirected, and each hyperedge contains unique node indices.
As equivariant GNNs (Section~\ref{sec:preliminary}) build upon higher-order tensors,
it is necessary to establish some connection between the hypergraphs and the higher-order tensors.

\subsection{Hypergraph as a Sequence of Higher-order Tensors}\label{sec:hypergraph_sequence}
\begin{figure}[!t]
    \centering
    \includegraphics[width=0.9\textwidth]{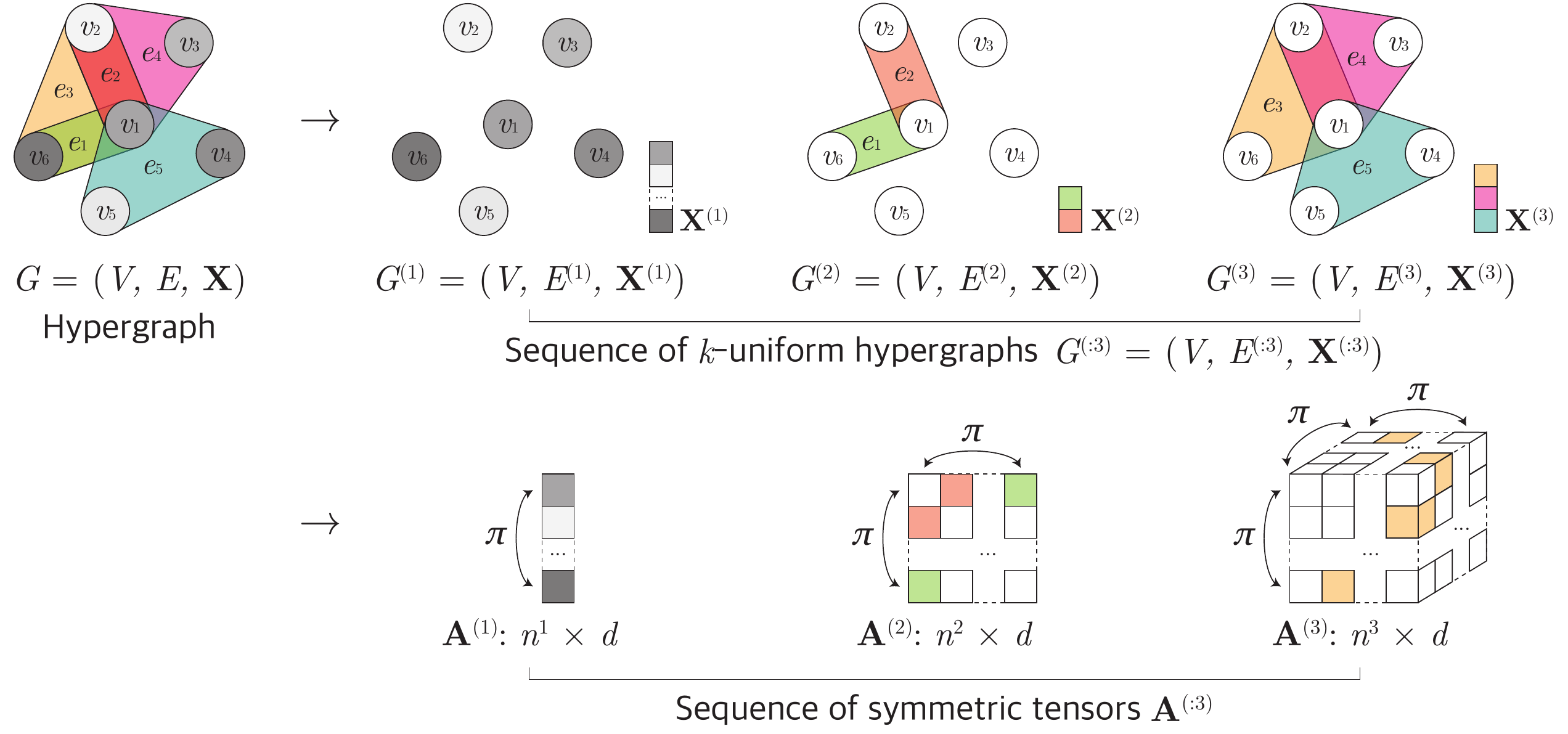}
    \caption{
    Example of a hypergraph represented as a sequence of $k$-uniform hypergraphs (Definition~\ref{defn:hypergraph_sequence}), or equivalently a sequence of symmetric higher-order tensors (Definition~\ref{defn:hypergraph_tensor_sequence}).
    Note that nodes are handled as first-order hyperedges.
    }
    \label{fig:hypergraph_sequence}
\end{figure}

To describe hypergraphs $(V, E, \mathbf{X})$ using higher-order tensors $\mathbf{A}\in\mathbb{R}^{n^k\times d}$, it is convenient to introduce \emph{$k$-uniform hypergraphs}.
A hypergraph is $k$-uniform if all of its hyperedges are exactly of order-$k$.
For example, a graph without self-loops is 2-uniform, and a triangle mesh is 3-uniform.
From that, we can define equivalent representation of a hypergraph as a \emph{sequence} of $k$-uniform hypergraphs:
\begin{defn}\label{defn:hypergraph_sequence}
The sequence representation of a hypergraph $(V, E, \mathbf{X})$ with max hyperedge order $K$ is a sequence of $k$-uniform hypergraphs with $k\leq K$, written as $(V, E^{(k)}, \mathbf{X}^{(k)})_{k\leq K}=(V, E^{(:K)}, \mathbf{X}^{(:K)})$ where $E^{(k)}$ as the set of all order-$k$ hyperedges in $E$ and $\mathbf{X}^{(k)}$ as a row stack of features $\{\mathbf{X}_e|e\in E^{(k)}\}$.
\end{defn}
As the collection $(E^{(k)})_{k\leq K}$ forms a partition of $E$, we can retrieve the original hypergraph $(V, E, \mathbf{X})$ from its sequence representation $(V, E^{(k)}, \mathbf{X}^{(k)})_{k\leq K}$ by using the union of $(E^{(k)})_{k\leq K}$ for $E$ and the concatenation of $(\mathbf{X}^{(k)})_{k\leq K}$ for $\mathbf{X}$.

The concept of uniform hypergraph is convenient because we can draw an equivalent representation as a \emph{symmetric} higher-order tensor~\cite{chien2022you, kofidis2002on}.
An order-$k$ tensor $\mathbf{A}$ is symmetric if its entries are invariant under reordering of indices, \emph{e.g.}, $\mathbf{A}_{ij}=\mathbf{A}_{ji}$, $\mathbf{A}_{ijk}=\mathbf{A}_{kij}=...$, and so on.
From that, we can define the equivalent representation of a $k$-uniform hypergraph as an order-$k$ symmetric tensor\footnote{Higher-order tensors can in principle represent directed hypergraphs as well; we constrain them to be symmetric to specifically represent undirected hypergraphs.}:
\begin{defn}\label{defn:hypergraph_tensor}
The tensor representation of $k$-uniform hypergraph $(V, E^{(k)}, \mathbf{X}^{(k)})$ is an order-$k$ symmetric tensor $\mathbf{A}^{(k)}\in\mathbb{R}^{n^k\times d}$ defined as follows:
\begin{align}\label{eqn:hypergraph_tensor}
    \mathbf{A}^{(k)}_{(i_1, ..., i_k)} = \left\{
        \begin{array}{cc}
            \mathbf{X}^{(k)}_e   &  \text{if }e=\{i_1, ..., i_k\}\in E^{(k)} \\
            0   &  {\textnormal{ otherwise}}
        \end{array}\right..
\end{align}
\end{defn}
From $\mathbf{A}^{(k)}$, we can retrieve the original $k$-uniform hypergraph $(V, E^{(k)}, \mathbf{X}^{(k)})$ by first identifying the indices of all nonzero entries of $\mathbf{A}^{(k)}$ to construct $E^{(k)}$, and then using $E^{(k)}$ to index $\mathbf{A}^{(k)}$ to construct $\mathbf{X}^{(k)}$.

Now, directly combining Definition~\ref{defn:hypergraph_sequence} and \ref{defn:hypergraph_tensor}, we can define the equivalent representation of a hypergraph as a sequence of higher-order tensors:
\begin{defn}\label{defn:hypergraph_tensor_sequence}
The tensor sequence representation of a hypergraph $(V, E, \mathbf{X})$ with maximum hyperedge order $K$ is a sequence of symmetric higher-order tensors $(\mathbf{A}^{(k)})_{k\leq K}=\mathbf{A}^{(:K)}$, where each $\mathbf{A}^{(k)}$ is the tensor representation (Definition~\ref{defn:hypergraph_tensor}) of each $k$-uniform hypergraph $(V, E^{(k)}, \mathbf{X}^{(k)})$ that comes from the sequence representation of the hypergraph $(V, E^{(k)}, \mathbf{X}^{(k)})_{k\leq K}=(V, E^{(:K)}, \mathbf{X}^{(:K)})$ (Definition~\ref{defn:hypergraph_sequence}).
\end{defn}
An illustration is in Fig.~\ref{fig:hypergraph_sequence}.
Note that we can include node features as $\mathbf{A}^{(1)}$.
Now, our problem of interest reduces to identifying a function $f$ that operates on sequences of tensors $\mathbf{A}^{(:K)}$ that represent hypergraphs.
The concept of permutation invariance and equivariance (Section~\ref{sec:preliminary}) applies similarly here.
A node permutation $\pi\in S_n$ acts on a tensor sequence $\mathbf{A}^{(:K)}$ by jointly acting on each tensor, $\pi\cdot\mathbf{A}^{(:K)}=(\pi\cdot\mathbf{A}^{(k)})_{k\leq K}$.
An invariant $f$ always satisfies $f(\pi\cdot \mathbf{A}^{(:K)})=f(\mathbf{A}^{(:K)})$, and an equivariant $f$ always satisfies $f(\pi\cdot \mathbf{A}^{(:K)})=\pi\cdot f(\mathbf{A}^{(:K)})$.

\subsection{Equivariant Linear Layers for Hypergraphs}\label{sec:equivariant_linear_layer_hypergraph}
In Definition~\ref{defn:hypergraph_tensor_sequence}, we represented a hypergraph as a sequence of symmetric higher-order tensors $(\mathbf{A}^{(k)})_{k\leq K}$, each tensor $\mathbf{A}^{(k)}$ representing a $k$-uniform hypergraph.
We now utilize the equivariant linear layers $L_{k\to l}:\mathbb{R}^{n^k\times d}\to\mathbb{R}^{n^l\times d'}$~(Eq.~\eqref{eqn:equivariant_linear_layer}) in Section~\ref{sec:preliminary} to formalize equivariant linear layers that input and output hypergraphs.
Our basic design is to find and combine all pairwise linear mappings between tensors (\emph{i.e.}, $k$-uniform hypergraphs) of input and output sequences.
Although seemingly simple, we prove that this gives the \emph{maximally expressive} equivariant linear layer for hypergraphs.

\subsubsection{Equivariant Linear Layers for $k$-uniform Hypergraphs}
In Section~\ref{sec:preliminary}, we argued that the equivariant linear layer $L_{k\to l}$ cannot be practically used due to the prohibitive number of $\text{b}(k+l)$ weights and $\text{b}(l)$ biases.
Yet, when the input and output tensors are restricted to $k$- and $l$-uniform hypergraphs respectively, we can show that the layer reduces to $\mathcal{O}(k+l)$ weights and a single bias:

\begin{figure}[!t]
    \centering
    \includegraphics[width=0.9\textwidth]{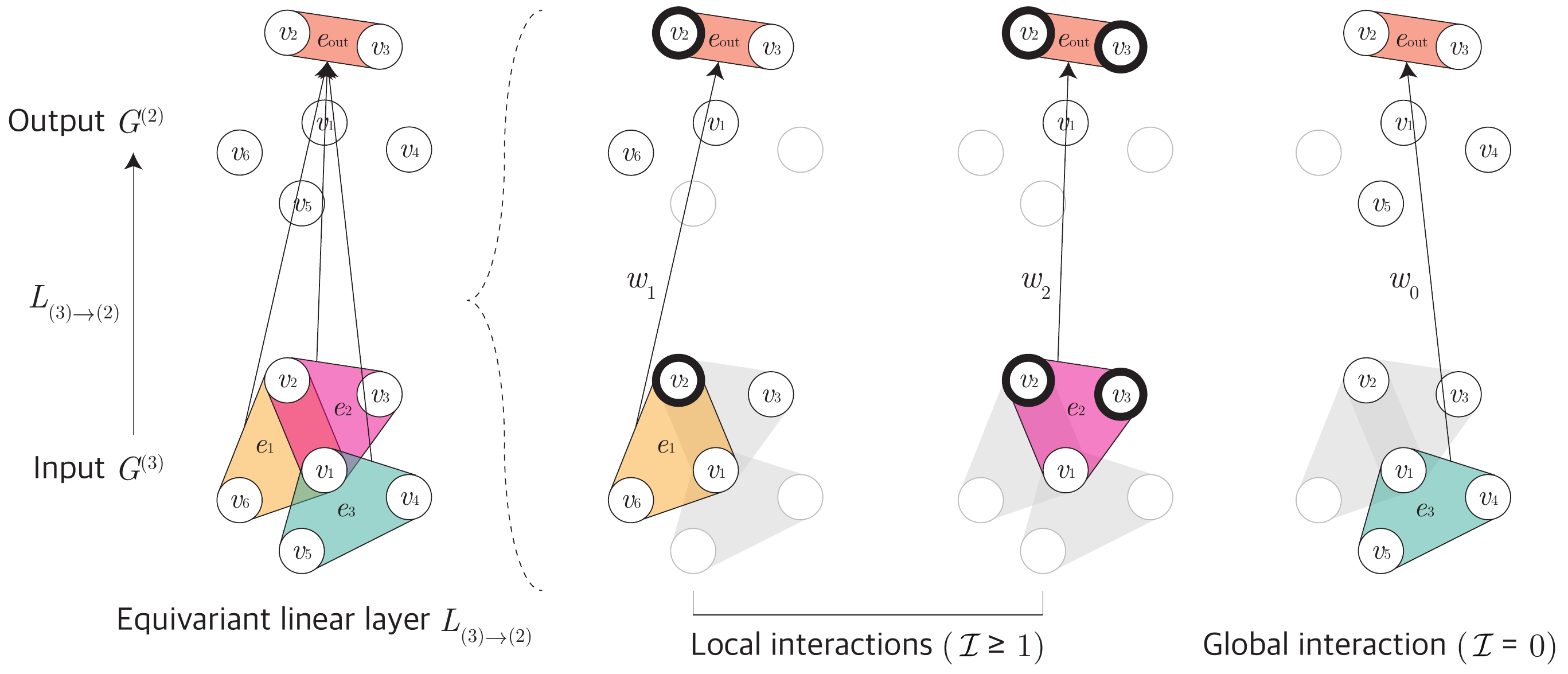}
    \caption{
    Conceptual illustration of an equivariant linear layer $L_{(3)\to(2)}$ as in Eq.~\eqref{eqn:equivariant_linear_layer_uniform}.
    The layer uses different weights $w_\mathcal{I}$ for different overlaps $\mathcal{I}$ between input and output hyperedges.
    This gives rise to local interactions similar to message-passing ($\mathcal{I} \geq 1$) and global interactions ($\mathcal{I} = 0$) implemented as global sum-pooling.
    }
    \label{fig:equivariant_layer_uniform}
    \vspace{-0.3cm}
\end{figure}

\begin{proposition}\label{proposition:equivariant_linear_layer_uniform}
Assume that the input and output of equivariant linear layer $L_{k\to l}$~(Eq.~\eqref{eqn:equivariant_linear_layer}) are constrained to symmetric tensors that represent $k$- and $l$-uniform hypergraphs respectively (Eq.~\eqref{eqn:hypergraph_tensor}).
Then it reduces to $L_{(k)\to (l)}$ below:
\begin{align}\label{eqn:equivariant_linear_layer_uniform}
    L_{(k)\to (l)}(\mathbf{A}^{(k)})_{\mathbf{j}} = \mathbbm{1}_{|\mathbf{j}|=l}\left(
    \sum_{\mathcal{I}=1}^{\textnormal{min}(k, l)} \sum_{\mathbf{i}}\mathbbm{1}_{|\mathbf{i}\cap\mathbf{j}|=\mathcal{I}} \mathbf{A}^{(k)}_{\mathbf{i}}w_\mathcal{I}
    + \sum_{\mathbf{i}}\mathbf{A}^{(k)}_{\mathbf{i}}w_0
    + b_l
    \right),
\end{align}
where $w_o, w_\mathcal{I}\in\mathbb{R}^{d\times d'}$, $b_l\in\mathbb{R}^{d'}$ are weight and bias, $|\mathbf{i}|$ is number of distinct elements in $\mathbf{i}$, and $|\mathbf{i}\cap\mathbf{j}|$ is number of distinct intersecting elements in $\mathbf{i}$ and $\mathbf{j}$.
\end{proposition}
The idea for the proof is that, if the input and output are constrained to tensors that represent uniform hypergraphs (Eq.~\eqref{eqn:hypergraph_tensor}), a large number of parameters are tied to adhere to the symmetry.
This leads to much fewer parameters compared to the original layer $L_{k\to l}$.
Still, $L_{(k)\to (l)}$~(Eq.~\eqref{eqn:equivariant_linear_layer_uniform}) is a maximally expressive linear layer as it produces identical outputs to (unreduced) $L_{k\to l}$.

Notably, Eq.~\eqref{eqn:equivariant_linear_layer_uniform} reveals that maximal expressiveness comprises sophisticated local message passing augmented with global interaction.
In the first term of Eq.~\eqref{eqn:equivariant_linear_layer_uniform}, the constraint $\mathbbm{1}_{|\mathbf{i}\cap\mathbf{j}|>0}$ specifies local dependency between \emph{incident} input and output hyperedges having at least one overlapping node.
This local interaction is more \emph{fine-grained} than conventional message passing, as it uses separate weights $w_{\mathcal{I}}$ for different numbers of overlapping nodes $\mathcal{I}$ (Fig.~\ref{fig:equivariant_layer_uniform}).
This is reminiscent of recent work in subgraph message passing~\cite{bevilacqua2022equivariant} that improves the expressive power of GNNs.
On top of that, the layer contains intrinsic global interaction via pooling in the second term of Eq.~\eqref{eqn:equivariant_linear_layer_uniform}, which reminds of the virtual node or global attention~\cite{li2017learning, ishiguro2019graph, knyazev2019understanding, louis2020global, puny2020from, wu2021representing} that also improves expressive power~\cite{puny2020from}.

\subsubsection{Equivariant Linear Layers for Hypergraphs}
We now construct maximally expressive equivariant linear layers for undirected hypergraphs.
As in Definition~\ref{defn:hypergraph_tensor_sequence}, a hypergraph can be represented as a sequence of tensors $(\mathbf{A}^{(k)})_{k\leq K}=\mathbf{A}^{(:K)}$.
Thus, we construct the linear layer $L_{(:K)\to (:L)}$ to input and output those tensor sequences while being equivariant $L_{(:K)\to (:L)}(\pi\cdot\mathbf{A}^{(:K)})=\pi\cdot L_{(:K)\to (:L)}(\mathbf{A}^{(:K)})$.
For this, we simply use all pairwise linear layers $L_{(k)\to(l)}$~(Eq.~\eqref{eqn:equivariant_linear_layer_uniform}) between tensors of input and output sequences:
\begin{align}\label{eqn:equivariant_linear_layer_hypergraph}
    L_{(:K)\to(:L)}(\mathbf{A}^{(:K)}) = \left(\sum_{k\leq K}L_{(k)\to (l)}(\mathbf{A}^{(k)})\right)_{l\leq L}.
\end{align}

For better interpretation, we plug Eq.~\eqref{eqn:equivariant_linear_layer_uniform} into Eq.~\eqref{eqn:equivariant_linear_layer_hypergraph} and rewrite it with respect to $\mathbf{j}$-th entry of $l$-th (order-$l$) output tensor:
\begin{align}\label{eqn:equivariant_linear_layer_hypergraph_indexed}
    L_{(:K)\to(:L)}(\mathbf{A}^{(:K)})_{l, \mathbf{j}} &= \mathbbm{1}_{|\mathbf{j}|=l}
    \sum_{k\leq K}\sum_{\mathcal{I}=1}^{\textnormal{min}(k, l)} \sum_{\mathbf{i}}\mathbbm{1}_{|\mathbf{i}\cap\mathbf{j}|=\mathcal{I}} \mathbf{A}^{(k)}_{\mathbf{i}}w_{k,l,\mathcal{I}}\nonumber\\
    &+ \mathbbm{1}_{|\mathbf{j}|=l}
    \sum_{k\leq K}\sum_{\mathbf{i}}\mathbf{A}^{(k)}_{\mathbf{i}}w_{k,l,0}
    + \mathbbm{1}_{|\mathbf{j}|=l}
    b_l.
\end{align}
Note that we added subscripts $(k, l)$ to $w_0$, $w_\mathcal{I}$ to differentiate between weights from each sublayer $L_{(k)\to(l)}$ as they are involved in different computations.
On the other hand, the biases from sublayers $(L_{(k)\to(l)})_{k\leq K}$ carry out precisely the same computation, and can be merged to a single bias $b_l$.
As a result, $L_{(:K)\to(:L)}$ contains $\sum_{l\leq L,k\leq K}(1 + \text{min}(k, l))$ weights and $L$ biases, achieving better scalability than the original $L_{K\to L}$ that has exponentially many weights and biases.

Similar to sublayers $L_{(k)\to(l)}$~(Eq.~\eqref{eqn:equivariant_linear_layer_uniform}), we see that the combined layer for general hypergraphs $L_{(:K)\to(:L)}$~(Eq.~\eqref{eqn:equivariant_linear_layer_hypergraph_indexed}) is a mixture of fine-grained local message passing and global interaction.
In this case, the local interactions utilize different weights $w_{k,l,\mathcal{I}}$ for each triplet $(k, l, \mathcal{I})$ that specifies dependency between order-$k$ input and order-$l$ output hyperedges with $\mathcal{I}$ overlapping nodes.
Similarly, global interactions (pooling) utilize different weights $w_{k, l, 0}$ for each pair $(k, l)$, specifying global dependency between all order-$k$ input and order-$l$ output hyperedges.
Finally, different biases $b_l$ are assigned for each output hyperedge order $l$.

Importantly, we can show the following:
\begin{theorem}\label{thm:equivariant_linear_layer_hypergraph}
$L_{(:K)\to(:L)}$~(Eq.~\eqref{eqn:equivariant_linear_layer_hypergraph}) is the maximally expressive equivariant linear layer for undirected hypergraphs represented as tensor sequences.
\end{theorem}
Similar as in Proposition~\ref{proposition:equivariant_linear_layer_uniform}, the idea for the proof is to appropriately constrain the input and output of the maximally expressive equivariant linear layer $L_{K\to L}$, and observe that most of its parameters are tied and reduced, leading to $L_{(:K)\to(:L)}$.
Still, the layer retains the maximal expressiveness of the original layer $L_{K\to L}$ as it produces the identical output.

\subsection{Equivariant Hypergraph Neural Networks (EHNN)}\label{sec:ehnn}
In Section~\ref{sec:equivariant_linear_layer_hypergraph}, we introduced equivariant linear layers for general undirected hypergraphs $L_{(:K)\to(:L)}$ by composing order-specific sublayers $L_{(k)\to(l)}$ for $k\leq K$, $l\leq L$ and proved their maximal expressiveness.
However, these layers are still unsuitable to be used in practice because they cannot input or output hypergraphs with orders exceeding $(K, L)$, and the number of weights and biases grows at least linearly to $(K, L)$ that can reach several hundred in practice.
To resolve the problems jointly, we propose \emph{Equivariant Hypergraph Neural Network} (EHNN) that introduces intrinsic trainable parameter sharing via \emph{hypernetworks}~\cite{ha2017hypernetworks}.
More specifically, we impose parameter sharing within $L_{(:K)\to(:L)}$ and across all sublayers $L_{(k)\to(l)}$ via two hypernetworks, each for weights and biases\footnote{Note that we do not share the hypernetworks across different levels of layers.}.
As a result, an EHNN layer is defined as follows, with hypernetworks $\mathcal{W}:\mathbb{N}^3\to\mathbb{R}^{d\times d'}$ and $\mathcal{B}:\mathbb{N}\to\mathbb{R}^{d'}$ inferring all weights $w_{k,l,\mathcal{I}}$ and biases $b_l$ (Eq.~\eqref{eqn:equivariant_linear_layer_hypergraph_indexed}) from the subscripts $(k,l,\mathcal{I})$ and $(l)$ respectively:
\begin{align}\label{eqn:ehnn}
    \text{EHNN}(\mathbf{A}^{(:K)})_{l, \mathbf{j}} &= \mathbbm{1}_{|\mathbf{j}|=l}
    \sum_{k\leq K}\sum_{\mathcal{I}=1}^{\textnormal{min}(k, l)} \sum_{\mathbf{i}}\mathbbm{1}_{|\mathbf{i}\cap\mathbf{j}|=\mathcal{I}} \mathbf{A}^{(k)}_{\mathbf{i}}\mathcal{W}(k,l,\mathcal{I})\nonumber\\
    &+ \mathbbm{1}_{|\mathbf{j}|=l}
    \sum_{k\leq K}\sum_{\mathbf{i}}\mathbf{A}^{(k)}_{\mathbf{i}}\mathcal{W}(k,l,0)
    + \mathbbm{1}_{|\mathbf{j}|=l}
    \mathcal{B}(l).
\end{align}
In principle, this preserves maximal expressiveness of $L_{(:K)\to(:L)}$ when $\mathcal{W}$ and $\mathcal{B}$ are parameterized as MLPs, as by universal approximation they can learn any lookup table that maps subscripts to weights and biases~\cite{hornik1989multilayer}.
Furthermore, as hypernetworks $\mathcal{W}$ and $\mathcal{B}$ can produce weights for arbitrary hyperedge orders $(k,l,\mathcal{I})$, we can remove the bound in hyperedge orders from the specification of the layer and use a single EHNN layer with bounded parameters to any hypergraphs with unbounded or unseen hyperedge orders.
Conclusively, EHNN layer is by far the first attempt that is maximally expressive (\emph{i.e.}, is able to model $L_{(:K)\to (:L)}$ (Theorem~\ref{thm:equivariant_linear_layer_hypergraph}) and thus exhausts the full space of equivariant linear layers on undirected hypergraphs) while being able to process arbitrary-order hypergraphs by construction.

\subsection{Practical Realization of EHNN}\label{sec:ehnn_realization}
The EHNN layer in Eq.~\eqref{eqn:ehnn} is conceptually elegant, but in practice it can be costly as we need to explicitly hold all output matrices of the hypernetwork $\mathcal{W}(k, l, \mathcal{I})\in\mathbb{R}^{d\times d'}$ in memory.
This motivates us to seek for simpler realizations of EHNN that can be implemented efficiently while retaining the maximal expressiveness.
To this end, we propose EHNN-MLP that utilizes three consecutive MLPs to approximate the role of the weight hypernetwork, and also propose its extension EHNN-Transformer with self-attention.
Then, we finish the section by providing a comparative analysis of EHNN-MLP and EHNN-Transformer with respect to the existing message-passing hypergraph neural networks.

\subsubsection{Realization with MLP}
We first introduce \emph{EHNN-MLP}, a simple realization of EHNN with three elementwise MLPs  $\phi_{1:3}$ where each $\phi_p:\mathbb{N}\times\mathbb{R}^{d_p}\to\mathbb{R}^{d_p'}$ takes a positive integer as an auxiliary input.
The intuition here is to \emph{decompose} the weight application with hypernetwork $\mathcal{W}(k, l, \mathcal{I})$ into three consecutive MLPs $\phi_1(k, \cdot)$, $\phi_2(\mathcal{I}, \cdot)$, and $\phi_3(l, \cdot)$, eliminating the need to explicitly store the inferred weights for each triplet $\mathcal{W}(k, l, \mathcal{I})$.
We characterize EHNN-MLP as follows:
\begin{align}
    \text{EHNN-MLP}
    (\mathbf{A}^{(:K)})_{l, \mathbf{j}} &=
    \phi_3\left(l,
    \sum_{\mathcal{I}\geq 0}
    \phi_2\left(\mathcal{I},
    \sum_{k\leq K} \sum_{\mathbf{i}}\mathbf{B}_{\mathbf{i},\mathbf{j}}^{\mathcal{I}}
    \phi_1(k, \mathbf{A}^{(k)}_{\mathbf{i}})
    \right)
    \right) + \mathcal{B}(l),\label{eqn:ehnn_mlp}\\
    \text{where }\mathbf{B}_{\mathbf{i},\mathbf{j}}^\mathcal{I} &= 
    \left\{\begin{array}{cc}
        \mathbbm{1}_{|\mathbf{i}\cap\mathbf{j}|=\mathcal{I}} & \scalebox{0.95}{if $\mathcal{I}\geq 1$} \\
        1 & \scalebox{0.95}{if $\mathcal{I} = 0$}
    \end{array}\right.,\label{eqn:ehnn_mlp_binary}
\end{align}
where we omit the output constraint $\mathbbm{1}_{|\mathbf{j}|=l}$ for brevity, and introduce a binary scalar $\mathbf{B}_{\mathbf{i},\mathbf{j}}^\mathcal{I}$ to write local ($\mathcal{I}\geq 1$) and global ($\mathcal{I} = 0$) interactions together.

Now we show that an EHNN-MLP layer can realize any EHNN layer:
\begin{theorem}\label{thm:ehnn_mlp}
An EHNN-MLP layer (Eq.~\eqref{eqn:ehnn_mlp}) can approximate any EHNN layer (Eq.~\eqref{eqn:ehnn}) to an arbitrary precision.
\end{theorem}
The proof is done by leveraging the universal approximation property~\cite{hornik1989multilayer} to model appropriate functions with the MLPs $\phi_{1:3}$, so that the output of EHNN-MLP~(Eq.~\eqref{eqn:ehnn_mlp}) accurately approximates the output of EHNN~(Eq.~\eqref{eqn:ehnn}).
As a result, with EHNN-MLP, we now have a practical model that can approximate the maximally expressive linear layer for general undirected hypergraphs.

In our implementation of the MLPs $\phi_{1:3}$, we first transform the input order ($k$, $l$ or $\mathcal{I}$) into a continuous vector called \emph{order embedding}, and combine it with the input feature through concatenation.
This way, the order embeddings are served similarly as the positional encoding used in Transformer~\cite{vaswani2017attention} with a subtle difference that it indicates the order of the input or output hyperedges.
We employ sinusoidal encoding~\cite{vaswani2017attention} to obtain order embedding due to its efficiency and, more importantly,
to aid extrapolation to unseen hyperedge orders in testing.

\subsubsection{Realization as a Transformer}
While EHNN-MLP (Eq.~\eqref{eqn:ehnn_mlp}) theoretically inherits the high expressive power of EHNN, in practice, its static sum-pooling can be limited in accounting for relative importance of input hyperedges.
A solution for this is to introduce more sophisticated pooling.
In particular, the attention mechanism of Transformers~\cite{vaswani2017attention} was shown to offer a large performance gain in set and (hyper)graph modeling~\cite{lee2019set, kim2021transformers, chien2022you, kim2022pure} via dynamic weighting of input elements.
Thus, we extend EHNN-MLP with multihead attention coefficients~$\boldsymbol{\alpha}_{\mathbf{i},\mathbf{j}}^{h,\mathcal{I}}$ and introduce \emph{EHNN-Transformer}, an advanced realization of EHNN:
\begin{align}
    &\text{Attn}(\mathbf{A}^{(:K)})_{l, \mathbf{j}} =
    \phi_3\left(l,
    \sum_{\mathcal{I}\geq 0}
    \phi_2\left(\mathcal{I},
    \sum_{h=1}^H\sum_{k\leq K} \sum_{\mathbf{i}}\boldsymbol{\alpha}_{\mathbf{i},\mathbf{j}}^{h,\mathcal{I}}
    \phi_1(k, \mathbf{A}^{(k)}_{\mathbf{i}})
    w_h^V
    \right)
    \right),\label{eqn:ehnn_transformer_attn}\\
    &\text{EHNN-Transformer}(\mathbf{A}^{(:K)}) = \text{Attn}(\mathbf{A}^{(:K)}) + \text{MLP}(\text{Attn}(\mathbf{A}^{(:K)})),\label{eqn:ehnn_transformer}
\end{align}
where we omit the output constraint $\mathbbm{1}_{|\mathbf{j}|=l}$ and bias $\mathcal{B}(l)$ for brevity. $H$ denotes the number of heads and $w_h^V\in\mathbb{R}^{d\times d_v}$ denotes the value weight matrix.
To compute attention coefficients $\boldsymbol{\alpha}_{\mathbf{i},\mathbf{j}}^{h,\mathcal{I}}$ from the input, we introduce additional query and key (hyper)networks $\mathcal{Q}:\mathbb{N}\to\mathbb{R}^{H\times d_H}$ and $\mathcal{K}:\mathbb{N}\times\mathbb{R}^d\to\mathbb{R}^{H\times d_H}$ and characterize scaled dot-product attention~\cite{vaswani2017attention} as follows:
\begin{align}\label{eqn:ehnn_transformer_attcoef}
    \boldsymbol{\alpha}_{\mathbf{i},\mathbf{j}}^{h,\mathcal{I}} &= 
    \left\{\begin{array}{lc}
        \sigma\left(
        \mathcal{Q}(\mathcal{I})_h
        \mathcal{K}\left(\mathcal{I},\phi_1(k, \mathbf{A}^{(k)}_{\mathbf{i}})\right)_h^\top
        /\sqrt{d_H}\cdot\mathbbm{1}_{|\mathbf{i}\cap\mathbf{j}|=\mathcal{I}}
        \right) & \scalebox{0.95}{if $\mathcal{I}\geq 1$} \\
        \sigma\left(
        \mathcal{Q}(0)_h
        \mathcal{K}\left(\mathcal{I},\phi_1(k, \mathbf{A}^{(k)}_{\mathbf{i}})\right)_h^\top
        /\sqrt{d_H}
        \right) & \scalebox{0.95}{if $\mathcal{I} = 0$}
    \end{array}\right.,
\end{align}
where $\sigma(\cdot)$ denotes activation, often chosen as softmax normalization.
Note that the query $\mathcal{Q}(\mathcal{I})$ is agnostic to output index $\mathbf{j}$, following prior works on set and (hyper)graph attention~\cite{lee2019set, chien2022you}.
Although this choice of attention mechanism has a drawback that assigning importance to input ($\mathbf{i}$) depending on output ($\mathbf{j}$) is not straightforward, we choose it in favor of scalability.

\subsubsection{Comparison to Message Passing Networks}
We finish the section by providing a comparative analysis of EHNNs with respect to the existing message passing networks for hypergraphs.
We specifically compare against \emph{AllSet}~\cite{chien2022you}, as it represents a highly general framework that subsumes most existing hypergraph neural networks.
Their MLP-based characterization \emph{AllDeepSets} can be written with two MLPs $\phi_1$ and $\phi_2$ as follows:
\begin{align}
    \text{AllDeepSets}
    (\mathbf{A}^{(:K)})_{l, \mathbf{j}} =
    \mathbbm{1}_{|\mathbf{j}|=l}
    \phi_2\left(
    \sum_{k\leq K} \sum_{\mathbf{i}}\mathbbm{1}_{|\mathbf{i}\cap\mathbf{j}|\geq 1}
    \phi_1(\mathbf{A}^{(k)}_{\mathbf{i}})
    \right).
    \label{eqn:alldeepsets}
\end{align}
We show the below by reducing EHNN-MLP to AllDeepSets through ablation:
\begin{theorem}\label{thm:alldeepsets}
An AllDeepSets layer (Eq.~\eqref{eqn:alldeepsets}) is a special case of EHNN-MLP layer (Eq.~\eqref{eqn:ehnn_mlp}), while the opposite is not true.
\end{theorem}
Finally, Theorem~\ref{thm:alldeepsets} leads to the following corollary: 
\begin{corollary}
An EHNN-MLP layer is more expressive than an AllDeepSets layer and also all hypergraph neural networks that AllDeepSets subsumes.
\end{corollary}
We provide an in-depth discussion including the comparison between EHNN-Transformer and AllSetTransformer~\cite{chien2022you} in Appendix~\ref{sec:apdx_in_depth_comparison_ehnn_allset}.

%% file: experiments.tex
\section{Experiments}\label{sec:experiments}
We test EHNN on a range of hypergraph learning problems, including synthetic node classification problem, real-world semi-supervised classification, and visual keypoint matching.
For the real-world tasks, we use 10 semi-supervised classification datasets used in Chien~et~al.~\cite{chien2022you} and two visual keypoint matching datasets used in Wang~et~al.~\cite{wang2019neural}.
Details including the datasets and hyperparameters can be found in Appendix~\ref{sec:apdx_experimental_details}.

\subsection{Synthetic $k$-edge Identification}\label{sec:k_edge_identification}
\begin{table}[t!]
\begin{center}
\tiny
\caption{
Results for synthetic $k$-edge identification.
We show averaged best test accuracy (\%) over 5 runs with standard deviation.
}
\label{table:hyperedge_identification}
\resizebox{0.8\textwidth}{!}{
\begin{tabular}{lccc}
\toprule
& Test involves only seen $k$ & \multicolumn{2}{c}{Test involves unseen $k$} \\
& & Interpolation & Extrapolation \\
\midrule
AllDeepSets & 76.99 ± 0.98 & 79.6 ± 0.86 & 79.01 ± 2.82 \\
AllSetTransformer & 77.61 ± 2.27 & 78.61 ± 2.367 & 77.35 ± 1.89 \\
\midrule
EHNN-MLP & 98.02 ± 0.73 & 90.70 ± 2.90 & 85.65 ± 2.89 \\
EHNN-Transformer & \textbf{99.69 ± 0.31} & \textbf{92.31 ± 1.47} & \textbf{90.19 ± 5.51} \\
\bottomrule
\end{tabular}
}
\end{center}
\end{table}

We devise a simple but challenging synthetic node classification task termed $k$-edge identification to demonstrate how the expressive power of EHNN can help learn complex hypergraph functions.
In an input hypergraph, we pick a random hyperedge and mark its nodes with a binary label.
The task is to identify all other nodes whose hyperedge order is as same as the marked one.
The model is required to propagate the information of marked hyperedge globally while also reasoning about the fine-grained structure of individual hyperedges for comparison.
We use 100 train and 20 test hypergraphs, each with 100 nodes and randomly wired 10 hyperedges of orders $\in\{2, ..., 10\}$.
To test generalization to unseen orders, we add two train sets where hyperedges are sampled \emph{excluding} orders $\{5, 6, 7\}$ (interpolation) or orders $\{8, 9, 10\}$ (extrapolation).
We evaluate the performance of EHNN-MLP/-Transformer with AllDeepSets and AllSetTransformer~\cite{chien2022you} as message passing baselines.
Further details can be found in Appendix~\ref{sec:apdx_experimental_details}.

The test performances are in Table~\ref{table:hyperedge_identification}.
EHNN achieves significant improvement over message passing nets, producing almost perfect predictions.
The result advocates that even for simple tasks there are cases where the high expressive power of a network is essential.
Furthermore, we observe evidence that the model can interpolate or even extrapolate to unseen hyperedge orders.
This supports the use of hypernetworks to infer parameters for potentially unseen orders.

\subsection{Semi-supervised Classification}
\begin{table}[p!]
\begin{center}
\caption{
Results for semi-supervised node classification.
Average accuracy (\%) over 20 runs are shown, and standard deviation can be found in Appendix~\ref{sec:apdx_experimental_details}.
Gray shade indicate the best result, and blue shade indicate results within one standard deviation of the best.
Baseline scores are taken from Chien~et.~al.~\cite{chien2022you}.
}
\label{table:semi_supervised}
\tiny
\resizebox{\textwidth}{!}{
\begin{tabular}{cccccccccccccccc}
\toprule
 & Zoo & 20Newsgroups & mushroom & NTU2012 & ModelNet40 & Yelp & House(1) & Walmart(1) & House(0.6) & Walmart(0.6) & avg. rank ($\downarrow$)\\
\midrule
MLP & 87.18 & \cellcolor{gray!60}\textbf{81.42} & \cellcolor{gray!60}\textbf{100.00} & 85.52 & 96.14 & 31.96 & 67.93 & 45.51 & 81.53 & 63.28 & 6.4\\
CEGCN & 51.54 & OOM & 95.27 & 81.52 & 89.92 & OOM & 62.80 & 54.44 & 64.36 & 59.78 & 11.5\\
CEGAT & 47.88 & OOM & 96.60 & 82.21 & 92.52 & OOM & \cellcolor{blue!30}69.09 & 51.14 & 77.25 & 59.47 & 10.5\\
HNHN & 93.59 & \cellcolor{blue!30}81.35 & \cellcolor{gray!60}\textbf{100.00} & \cellcolor{blue!30}89.11 & 97.84 & 31.65 & 67.80 & 47.18 & 78.78 & 65.80 & 5.9\\
HGNN & 92.50 & 80.33 & 98.73 & 87.72 & 95.44 & 33.04 & 61.39 & 62.00 & 66.16 & 77.72 & 7.8\\
HCHA & 93.65 & 80.33 & 98.70 & 87.48 & 94.48 & 30.99 & 61.36 & 62.45 & 67.91 & 77.12 & 8.1\\
HyperGCN & N/A & \cellcolor{blue!30}81.05 & 47.90 & 56.36 & 75.89 & 29.42 & 48.31 & 44.74 & 78.22 & 55.31 & 12.4\\
UniGCNII & 93.65 & \cellcolor{blue!30}81.12 & 99.96 & \cellcolor{blue!30}89.30 & 98.07 & 31.70 & 67.25 & 54.45 & 80.65 & 72.08 & 5.8\\
HAN (full batch) & 85.19 & OOM & 90.86 & 83.58 & 94.04 & OOM & \cellcolor{blue!30}71.05 & OOM & \cellcolor{blue!30}83.27 & OOM & 9.9\\
HAN (minibatch) & 75.77 & 79.72 & 93.45 & 80.77 & 91.52 & 26.05 & 62.00 & 48.57 & 82.04 & 63.1 & 10.6\\
\midrule
AllDeepSets & \cellcolor{blue!30}95.39 & \cellcolor{blue!30}81.06 & 99.99 & 88.09 & 96.98 & 30.36 & 67.82 & 64.55 & 80.70 & 78.46 & 5.4\\
AllSetTransformer & \cellcolor{gray!60}\textbf{97.50} & \cellcolor{blue!30}81.38 & \cellcolor{gray!60}\textbf{100.00} & \cellcolor{blue!30}88.69 & \cellcolor{blue!30}98.20 & \cellcolor{gray!60}\textbf{36.89} & \cellcolor{blue!30}69.33 & 65.46 & \cellcolor{blue!30}83.14 & 78.46 & 2.4\\
\midrule
EHNN-MLP & 91.15 & \cellcolor{blue!30}81.31 & 99.99 & 87.35 & 97.74 & 35.80 & 67.41 & 65.65 & 82.29 & 78.80 & 5.0\\
EHNN-Transformer & 93.27 & \cellcolor{gray!60}\textbf{81.42} & \cellcolor{gray!60}\textbf{100.00} & \cellcolor{gray!60}\textbf{89.60} & \cellcolor{gray!60}\textbf{98.28} & \cellcolor{blue!30}36.48 & \cellcolor{gray!60}\textbf{71.53} & \cellcolor{gray!60}\textbf{68.73} & \cellcolor{gray!60}\textbf{85.09} & \cellcolor{gray!60}\textbf{80.05} & \cellcolor{gray!60}\textbf{1.6}\\
\bottomrule
\end{tabular}
}
\end{center}
\end{table}

To test EHNN in real-world hypergraph learning, we use 10 transductive semi-supervised node classification datasets~\cite{chien2022you}.
The data is randomly split into 50\% training, 25\% validation, and 25\% test.
We run the experiment 20 times with random splits and initialization and report aggregated performance.

The test performances are in Table~\ref{table:semi_supervised}.
Our methods often achieve favorable scores over strong baselines, \emph{e.g.}, AllDeepSets and AllSetTransformer --  our models improve the state-of-the-art by $3.27\%$ in Walmart~(1), $1.82\%$ in House~(0.6), and $1.54\%$ in Walmart~(0.6).
Notably, EHNN-Transformer gives a state-of-the-art performance in most cases.
This supports the notion that attention strengthens equivariant networks~\cite{kim2021transformers,lee2019set}, and also implies that the high expressiveness of EHNN makes it strong in general hypergraph learning setups involving not only social networks but also vision and graphics (NTU2012 and ModelNet40).

\begin{table}[p!]
\begin{center}
\tiny
\caption{Hypergraph matching accuracy (\%) on Willow test set.}
\label{table:matching_willow}
\begin{tabular}{cccccc|c}
\toprule
 & car & duck & face & motor. & wine. & avg.\\
\midrule
GMN      & 38.85 & 38.75 & 78.85 & 28.08 & 45.00 & 45.90\\
NGM      & 77.50 & 85.87 & 99.81 & 77.50 & 89.71 & 86.08\\
NHGM     & 69.13 & 83.08 & 99.81 & 73.37 & 88.65 & 82.81\\
NMGM     & 74.95 & 81.33 & 99.83 & 78.26 & 92.06 & 85.29\\
IPCA-GM  & 79.58 & 80.20 & 99.70 & 73.37 & 83.75 & 83.32\\
CIE-H    & 9.37 & 8.87 & 9.88 & 11.84 & 9.84 & 9.96\\
BBGM     & 96.15 & 90.96 & \textbf{100.00} & 96.54 & \textbf{99.23} & 96.58\\
GANN-MGM & 92.11 & 90.11 & \textbf{100.00} & 96.21 & 98.26 & 95.34\\
\midrule
NGM-v2   & 94.81 & 89.04 & \textbf{100.00} & 96.54 & 95.87 & 95.25\\
NHGM-v2  & 89.33 & 83.17 & \textbf{100.00} & 92.60 & 95.96 & 92.21\\
\midrule
EHNN-MLP  & 94.71 & 91.92 & \textbf{100.00} & 97.21 & 97.79 & 96.33\\
EHNN-Transformer & \textbf{97.02} & \textbf{92.69} & \textbf{100.00} & \textbf{97.60} & 98.08 & \textbf{97.08} \\
\bottomrule
\end{tabular}
\end{center}
\end{table}

\begin{table}[p!]
\tiny
\begin{center}
\caption{Hypergraph matching accuracy (\%) on PASCAL-VOC test set.}
\label{table:matching_pascal}
\begin{tabular}{cccccccccccc}
\toprule
 & aero & bike & bird & boat & botl & bus & car & cat & chair & cow & desk \\
 \midrule
GMN      & 40.67 & 57.62 & 58.19 & 51.38 & 77.55 & 72.48 & 66.90 & 65.04 & 40.43 & 61.56 & 65.17\\
PCA-GM   & 51.46 & 62.43 & 64.70 & 58.56 & 81.94 & 75.18 & 69.56 & 71.05 & 44.53 & 65.81 & 39.00\\
NGM      & 12.09 & 10.01 & 17.44 & 21.73 & 12.03 & 21.40 & 20.16 & 14.26 & 15.10 & 12.07 & 14.50\\
NHGM     & 12.09 & 10.01 & 17.44 & 21.73 & 12.03 & 21.40 & 20.16 & 14.26 & 15.10 & 12.07 & 14.50\\
IPCA-GM  & 50.78 & 62.29 & 63.87 & 58.94 & 79.46 & 74.18 & 72.60 & 71.52 & 41.42 & 64.12 & 36.67\\
CIE-H    & 52.26 & 66.79 & 69.09 & 59.76 & 83.38 & 74.61 & 69.93 & 71.04 & 43.36 & 69.20 & 76.00\\
BBGM     & \textbf{60.06} & 71.32 & 78.21 & 78.97 & 88.63 & \textbf{95.57} & \textbf{89.52} & 80.53 & \textbf{59.34} & \textbf{77.80} & 76.00\\
GANN-MGM & 14.75 & 32.20 & 21.31 & 24.43 & 67.23 & 36.35 & 21.09 & 17.20 & 25.73 & 21.00 & 37.50\\
\midrule
NGM-v2   & 42.88 & 61.70 & 63.63 & 75.62 & 84.66 & 90.58 & 75.34 & 72.26 & 44.42 & 66.67 & 74.50\\
NHGM-v2  & 57.04 & 71.88 & 76.06 & \textbf{79.96}& \textbf{89.79} & 93.70 & 86.16 & 80.76 & 56.36 & 76.70 & 74.33\\
\midrule
EHNN-MLP  & 57.34 & \textbf{73.89} & 76.41 & 78.41 & 89.40 & 94.51 & 85.58 & 79.83 & 56.39 & 76.56 & \textbf{91.00}\\
EHNN-Transformer & 60.04 & 72.36 & \textbf{78.25} & 78.59 & 87.61 & 93.77 & 87.99 & \textbf{80.78} & 58.76 & 76.29 & 81.17\\
 \midrule
 & dog & horse & mbk & prsn & plant & sheep & sofa & train & tv & \multicolumn{2}{|c}{avg.} \\
 \midrule
GMN      & 61.56 & 62.18 & 58.96 & 37.80 & 78.39 & 66.89 & 39.74 & 79.84 & 90.94 & \multicolumn{2}{|c}{61.66}\\
PCA-GM   & 67.82 & 65.18 & 65.71 & 46.21 & 83.81 & 70.51 & 49.88 & 80.87 & 93.07 & \multicolumn{2}{|c}{65.36}\\
NGM      & 12.83 & 12.05 & 15.69 & 09.76 & 21.00 & 17.10 & 15.12 & 31.11 & 24.88 & \multicolumn{2}{|c}{16.52}\\
NHGM     & 12.83 & 12.05 & 15.67 & 09.76 & 21.00 & 17.10 & 14.66 & 31.11 & 24.83 & \multicolumn{2}{|c}{16.49}\\
IPCA-GM  & 69.11 & 66.05 & 65.88 & 46.97 & 83.09 & 68.97 & 51.83 & 79.17 & 92.27 & \multicolumn{2}{|c}{64.96}\\
CIE-H    & 69.68 & 71.18 & 66.14 & 46.76 & 87.22 & 71.08 & 59.16 & 82.84 & 92.60 & \multicolumn{2}{|c}{69.10}\\
BBGM     & \textbf{80.39} & 77.80 & 76.48 & \textbf{65.99} & \textbf{98.52} & \textbf{78.07} & 76.65 & 97.61 & 94.36 & \multicolumn{2}{|c}{\textbf{80.09}}\\
GANN-MGM & 16.16 & 20.16 & 25.92 & 19.20 & 53.76 & 18.34 & 26.16 & 46.30 & 72.32 & \multicolumn{2}{|c}{30.85}\\
\midrule
NGM-v2   & 67.83 & 68.92 & 68.86 & 47.40 & 96.69 & 70.57 & 70.01 & 95.13 & 92.49 & \multicolumn{2}{|c}{71.51}\\
NHGM-v2  & 76.75 & 77.45 & \textbf{76.81} & 58.56 & 98.21 & 75.34 & 76.42 & 98.10 & 94.80 & \multicolumn{2}{|c}{78.76}\\
\midrule
EHNN-MLP & 76.57 & \textbf{78.65} & 75.54 & 58.92 & 98.31 & 76.53 & \textbf{81.14} & 98.08 & \textbf{95.01} & \multicolumn{2}{|c}{79.90}\\
EHNN-Transformer & 78.30 & 76.91 & 75.79 & 63.78 & 97.60 & 76.47 & 78.04 & \textbf{98.53} & 93.83 & \multicolumn{2}{|c}{79.74}\\
\bottomrule
\end{tabular}
\end{center}
\end{table}

\subsection{Visual Keypoint Matching}
To test EHNN in computer vision problems represented as hypergraph learning, we tackle visual keypoint matching.
The task is considered challenging due to the discrepancy between the two images in terms of viewpoint, scale, and lighting.
Following previous work~\cite{wang2019neural}, we view the problem as \emph{hypergraph matching}, where keypoints of each image form a hypergraph.
This is considered helpful as the hyperedge features can capture rotation- and scale-invariant geometric features such as angles.
We then cast hypergraph matching to binary node classification on a single association hypergraph as in previous work~\cite{wang2019neural}.

We use two standard datasets~\cite{wang2019neural}: Willow ObjectClass~\cite{cho2013learning} and PASCAL-VOC~\cite{everingham2010the, bourdev2009poselets}.
The Willow dataset consists of 256 images with 5 object categories.
The PASCAL-VOC dataset contains 11,530 images with 20 object categories and is considered challenging due to the large variance in illumination and pose.
We follow the training setup of NHGM-v2~\cite{wang2019neural} and only replace the hypergraph neural network module with EHNN-MLP/-Transformer.
The key difference between NHGM-v2 and our models is that NHGM-v2 utilizes two \emph{separate} message passing networks, one on 2-edges and another on 3-edges, and aggregates node features as a weighted sum.
In contrast, EHNN \emph{mixes} the information from 2-edges and 3-edges extensively via shared MLP hypernetworks and global interactions.
Further details including datasets and training setups are in Appendix~\ref{sec:apdx_experimental_details}.

The results are in Tables~\ref{table:matching_willow}~and~\ref{table:matching_pascal}.
On Willow, EHNN-Transformer gives the best performance, improving over NHGM-v2 by $4.87\%$.
On PASCAL-VOC, EHNNs improve over NHGM-v2 by $\sim1\%$ and are competitive to the best (BBGM; $0.19\%$ gap) that relies on a sophisticated combinatorial solver~\cite{rolinek2020deep}.
We conjecture that intrinsic and global mixing of 2-edge (distance) and 3-edge (angle) features improves hypergraph learning, and consequently benefits keypoint matching.

\subsection{Ablations}

\begin{table}[t!]
\tiny
\begin{center}
\caption{Ablation study on $k$-edge identification.
We show averaged best test accuracy (\%) over 5 runs with standard deviation.
The results for AllDeepSets and EHNN-MLP are taken from Table~\ref{table:hyperedge_identification}.}
\label{table:ablation}
\resizebox{\textwidth}{!}{
\begin{tabular}{lccc|ccc}
\toprule
& & & & Seen $k$ & \multicolumn{2}{c}{Unseen $k$} \\
& Global interaction & Order emb. & MLP realization & & Interpolation & Extrapolation \\
\midrule
AllDeepSets & $\times$ & $\times$ & $\bigcirc$ & 76.99±0.98 & 79.6±0.86 & 79.01±2.82 \\
\midrule
EHNN-MLP (ablated) & & & & & & \\
\ \ $\bullet$ w/o global and order & $\times$ & $\times$ & $\bigcirc$ & 78.94±1.18 & 78.6±1.66 & 77.97±1.83 \\
\ \ $\bullet$ w/o global & $\times$ & $\bigcirc$ & $\bigcirc$ & 80.34±2.87 & 77.86±2.38 & 79.56±3.03 \\
\ \ $\bullet$ w/o order & $\bigcirc$ & $\times$ & $\bigcirc$ & 84.05±1.77 & 81.30±3.56 & 80.17±3.34 \\
\midrule
EHNN (na\"ive) & & & & & & \\
\ \ $\bullet$ Lookup table for $\mathcal{W},\mathcal{B}$ & $\bigcirc$ & $\bigcirc$ & $\times$ & 87.09±2.49 & 84.09±1.29 & 80.20±2.21 \\
\ \ $\bullet$ Hypernetwork for $\mathcal{W},\mathcal{B}$ & $\bigcirc$ & $\bigcirc$ & $\times$ & 83.74±2.89 & 83.19±1.57 & 79.93±2.61 \\
\midrule
EHNN-MLP & $\bigcirc$ & $\bigcirc$ & $\bigcirc$ & $\mathbf{98.02}$±$\mathbf{0.73}$ & $\mathbf{90.70}$±$\mathbf{2.90}$ & $\mathbf{85.65}$±$\mathbf{2.89}$ \\
\bottomrule
\end{tabular}
}
\end{center}
\end{table}

To further identify the source of performance improvements, we perform a comprehensive ablation study on $k$-edge identification (Section~\ref{sec:k_edge_identification}).
We gradually ablate each component of EHNN-MLP~(Section~\ref{sec:ehnn_realization}), specifically the order conditioning in the elementwise MLPs~(Eq.~\eqref{eqn:ehnn_mlp}) and the global interactions~(Eq.~\eqref{eqn:ehnn_mlp_binary}), until it reduces to message passing ($\approx$AllDeepSets~\cite{chien2022you}).
We also compare against na\"ive EHNN (Section~\ref{sec:ehnn}) which is maximally expressive but, unlike EHNN-MLP, is not realized as 3 elementwise MLPs.
The results are in Table~\ref{table:ablation}.
We find ablating any component of EHNN-MLP degrades performance until similar to AllDeepSets.

\subsection{Time and Memory Cost Analysis}

\begin{table}[t!]
\tiny
\begin{center}
\caption{Runtime and memory cost analysis. We show aggregated results over 20 runs on a single A100 GPU.}
\label{table:cost}
\resizebox{0.8\textwidth}{!}{
\begin{tabular}{lccc}
\toprule
& Forward (ms) & Backward (ms) & Peak mem. (MB) \\
\midrule
AllDeepSets & 5.538±0.689 & 3.470±0.071 & 4.608±0.000 \\
AllSetTransformers & 6.883±0.657 & 5.037±0.681 & 5.077±0.000 \\
\midrule
EHNN (na\"ive) & & \\
\ \ $\bullet$ Lookup table for $\mathcal{W},\mathcal{B}$ & 28.68±0.535 & 71.71±3.597 & 14.43±0.000 \\
\ \ $\bullet$ Hypernetwork for $\mathcal{W},\mathcal{B}$ & 31.56±1.110 & 76.78±6.044 & 17.40±0.000 \\
\midrule
EHNN-MLP & 7.517±0.865 & 7.179±0.548 & 7.361±0.000 \\
EHNN-Transformer & 14.51±1.146 & 13.41±0.0962 & 11.98±0.000 \\
\bottomrule
\end{tabular}
}
\end{center}
\end{table}

To test the computational efficiency of EHNN, we perform runtime and memory cost analysis using random hypergraphs with $1024$ nodes and $128$ randomly wired hyperedges with orders $\in\{2, ..., 10\}$.
The results are outlined in Table~\ref{table:cost}.
Na\"ive~EHNN (Section~\ref{sec:ehnn}) suffers from high time and memory cost as it requires to explicitly hold all order-specific weight matrices in memory.
In contrast, EHNN-MLP/Transformer (Section~\ref{sec:ehnn_realization}) are significantly more efficient, improving time and memory cost from $5-20\times$ to $2-3\times$ with respect to highly optimized message passing while still being maximally expressive.

%% file: conclusion.tex
\section{Conclusion}
We proposed a family of hypergraph neural networks coined Equivariant Hypergraph Neural Network (EHNN).
EHNN extends the theoretical foundations of equivariant GNNs to general undirected hypergraphs by representing a hypergraph as a sequence of tensors and combining equivariant linear layers on them.
We further proposed EHNN-MLP/-Transformer, practical realizations of EHNN based on MLP hypernetworks.
We show that EHNN is theoretically more expressive than most message passing networks and provide empirical evidence.

%% file: appendix.tex
\appendix
\section{Appendix}\label{sec:apdx}
\subsection{Proofs}\label{sec:apdx_proofs}
\subsubsection{Extended Preliminary (Cont. from Section~\ref{sec:preliminary})}
Before presenting the proofs, we provide extended preliminary on the maximally expressive equivariant linear layers in Maron~et~al.~(2019)~\cite{maron2019invariant} to supplement Section~\ref{sec:preliminary}.
Given an order-$k$ input tensor $\mathbf{A}\in\mathbb{R}^{n^k\times d}$, the order-$l$ output tensor of an equivariant linear layer $L_{k\to l}$ is written as below, with indicator $\mathbbm{1}$ and multi-indices $\mathbf{i}\in[n]^k$, $\mathbf{j}\in[n]^l$:
\begin{align}\label{eqn:equivariant_linear_layer_recap}
    L_{k\to l}(\mathbf{A})_\mathbf{j} = \sum_\mu\sum_\mathbf{i}\mathbbm{1}_{(\mathbf{i},\mathbf{j})\in\mu}\mathbf{A}_\mathbf{i}w_\mu + \sum_\lambda\mathbbm{1}_{\mathbf{j}\in\lambda}b_\lambda,
\end{align}
where $w_\mu\in\mathbb{R}^{d\times d'}$, $b_\lambda\in\mathbb{R}^{d'}$ are weights and biases, and $\mu$ and $\lambda$ are equivalence classes of order-$(k+l)$ and order-$l$ multi-indices, respectively.
As briefly introduced in Section~\ref{sec:preliminary}, the equivalence classes specify a partitioning of a multi-index space; the equivalence classes $\mu$ partition the space of order-$(k+l)$ multi-indices $[n]^{k+l}$, and the equivalence classes $\lambda$ partition the space of order-$l$ multi-indices $[n]^l$.

More specifically, the equivalence classes are defined upon equivalence relation~$\sim$ of multi-indices that partitions the multi-index space~$[n]^k$ into $[n]^k/_\sim$.
The equivalence relation $\sim$ is defined as follows: for $\mathbf{i},\mathbf{j}\in[n]^k$, the equivalence relation sets $\mathbf{i}\sim\mathbf{j}$ if and only if $(i_1, ..., i_k)=(\pi(j_1), ..., \pi(j_k))$ holds for some node permutation $\pi\in S_n$.
As a result, a multi-index $\mathbf{i}$ and all elements $\mathbf{j}$ of its equivalence class $\mu$ have an identical (permutation-invariant) equality pattern: $\mathbf{i}_a=\mathbf{i}_b\Leftrightarrow\mathbf{j}_a=\mathbf{j}_b$ for all $\mathbf{i},\mathbf{j}\in\mu$.
Each equivalence class $\mu$ (or $\lambda$) in Eq.~\eqref{eqn:equivariant_linear_layer_recap} is thus described as a specific set of all order-($k+l$) (order-$l$) multi-indices with an identical (permutation-invariant) equality pattern.

\subsubsection{Equivariant Linear Layers for Symmetric Tensors}
For the proof of Proposition~\ref{proposition:equivariant_linear_layer_uniform} and Theorem~\ref{thm:equivariant_linear_layer_hypergraph} (Section~\ref{sec:equivariant_linear_layer_hypergraph}), it is convenient that we first derive the maximally expressive equivariant linear layers for \emph{symmetric} input and output tensors.
This is because all tensors under consideration in this work are symmetric due to the unorderedness of hypergraphs (see Definition~\ref{defn:hypergraph_tensor}).

In the following Lemma, we show that if input and output of $L_{k\to l}$ (Eq.~\eqref{eqn:equivariant_linear_layer_recap}) are constrained to be \emph{symmetric}, the layer is described with \emph{coarser} partitioning of index spaces (in terms of partition refinement) specified by equivalence relations that are invariant to axis permutation on top of node permutation:
\begin{lemma}\label{lemma:equivariant_linear_layer_symmetric}
If input and output of an equivariant linear layer $L_{k\to l}$ (Eq.~\eqref{eqn:equivariant_linear_layer}) are constrained to be symmetric tensors, it can be reduced to the following $L_{[k]\to[l]}$:
\begin{align}\label{eqn:equivariant_linear_layer_symmetric}
    L_{[k]\to [l]}(\mathbf{A})_\mathbf{j} = \sum_\alpha\sum_\mathbf{i}\mathbbm{1}_{(\mathbf{i},\mathbf{j})\in\alpha}\mathbf{A}_\mathbf{i}w_\alpha + \sum_\beta\mathbbm{1}_{\mathbf{j}\in\beta}b_\beta,
\end{align}
where $\alpha$ and $\beta$ are equivalence classes that specify the partitioning $[n]^{k+l}/_{\sim_\alpha}$ and $[n]^l/_{\sim_\beta}$ respectively, defined as the following:
\begin{enumerate}
    \item The equivalence classes $\alpha$ are defined upon the equivalence relation $\sim_\alpha$ that, for $\mathbf{i}, \mathbf{j}\in[n]^{k+l}$, relates $\mathbf{i}\sim_\alpha\mathbf{j}$ if and only if the following holds for some node permutation $\pi\in S_n$ and axis permutations $\pi_k\in S_k$, $\pi_l\in S_l$:
    \begin{align}\label{eqn:weight_indices_symmetry}
        (i_1, ..., i_{k+l}) = \left(\pi(j_{\pi_k(1)}), ..., \pi(j_{\pi_k(k)}), \pi(j_{k+\pi_l(1)}), ..., \pi(j_{k+\pi_l(l)})\right).
    \end{align}
    \item The equivalence classes $\beta$ are defined upon the equivalence relation $\sim_\beta$ that, for $\mathbf{i}, \mathbf{j}\in[n]^{l}$, relates $\mathbf{i}\sim_\beta\mathbf{j}$ if and only if the following holds for some node permutation $\pi\in S_n$ and axis permutation $\pi_l\in S_l$:
    \begin{align}\label{eqn:bias_indices_symmetry}
        (i_1, ..., i_l) = \left(\pi(j_{\pi_l(1)}), ..., \pi(j_{\pi_l(l)})\right).
    \end{align}
\end{enumerate}
\end{lemma}
\begin{proof}
We begin from $L_{k\to l}$ (Eq.~\eqref{eqn:equivariant_linear_layer_recap}) and reduce its parameters without affecting the output based on symmetry of input and output.
For $\pi_k\in S_k$, we denote \emph{axis permutation} of a multi-index $\mathbf{i}\in[n]^k$ as $\pi_k(i_1, ..., i_k)=(i_{\pi_k(1)}, ..., i_{\pi_k(k)})$ and denote axis permutation of a tensor $\mathbf{A}\in\mathbb{R}^{n^k\times d}$ as $(\pi_k \cdot \mathbf{A})_\mathbf{i}=\mathbf{A}_{\pi_k^{-1}(\mathbf{i})}$.
With symmetry, input and output of $L_{k\to l}$ are constrained to be axis permutation invariant, \emph{i.e.}, the following holds:
\begin{align}\label{eqn:equivariant_linear_layer_axis_permutation_invariance}
    L_{k\to l}(\mathbf{A}) = \pi_l\cdot L_{k\to l}(\pi_k\cdot\mathbf{A}),
\end{align}
for all $\pi_k\in S_k$, $\pi_l\in S_l$, and symmetric $\mathbf{A}\in\mathbb{R}^{n^k\times d}$.
In other words:
\begin{align}
    L_{k\to l}(\mathbf{A})_\mathbf{j} = \sum_\mu\sum_\mathbf{i}\mathbbm{1}_{(\mathbf{i},\pi_l^{-1}(\mathbf{j}))\in\mu}\mathbf{A}_{\pi_k^{-1}(\mathbf{i})}w_\mu + \sum_\lambda\mathbbm{1}_{\pi_l^{-1}(\mathbf{j})\in\lambda}b_\lambda,
\end{align}
which, denoting $\pi_l^{-1}$ as $\pi_l$ and noting $\pi_k$ is a bijection on $[n]^k$, leads to following:
\begin{align}\label{eqn:equivariant_linear_layer_axis_permutation}
    L_{k\to l}(\mathbf{A})_\mathbf{j} = \sum_\mu\sum_\mathbf{i}\mathbbm{1}_{(\pi_k(\mathbf{i}),\pi_l(\mathbf{j}))\in\mu}\mathbf{A}_{\mathbf{i}}w_\mu + \sum_\lambda\mathbbm{1}_{\pi_l(\mathbf{j})\in\lambda}b_\lambda.
\end{align}

We now reduce biases.
From Eq.~\eqref{eqn:equivariant_linear_layer_axis_permutation}, we see that by constraining the output to be symmetric, $\sum_\lambda\mathbbm{1}_{\mathbf{j}\in\lambda}b_\lambda = \sum_\lambda\mathbbm{1}_{\pi_l(\mathbf{j})\in\lambda}b_\lambda$ holds for all $\mathbf{j}\in[n]^l$ and $\pi_l\in S_l$.
We can see that this holds if and only if $b_{\lambda_1}=b_{\lambda_2}$ for all ($\lambda_1,\lambda_2$) such that, for some $\mathbf{j}\in{\lambda_1}$, $\pi_l(\mathbf{j})\in{\lambda_2}$ holds for some $\pi_l\in S_l$.
By writing $b_\beta=b_{\lambda_1}=b_{\lambda_2}$, we can interpret $\beta$ as a new equivalence class of multi-indices formed as a union of all such $\{\lambda_1, \lambda_2, ...\}$ (thereby specifying a coarser partitioning than $\sim$), or equivalently, by collecting all multi-indices related by node and axis permutations, \emph{i.e.}, equivalence relation $\sim_\beta$.
Thus, we can reduce biases as $\sum_\beta\mathbbm{1}_{\mathbf{j}\in\beta}b_\beta$.

We now reduce weights.
From Eq.~\eqref{eqn:equivariant_linear_layer_axis_permutation}, we see that by constraining the input and output to be symmetric, $\sum_\mu\sum_\mathbf{i}\mathbbm{1}_{(\mathbf{i},\mathbf{j})\in\mu}\mathbf{A}_{\mathbf{i}}w_\mu = \sum_\mu\sum_\mathbf{i}\mathbbm{1}_{(\pi_k(\mathbf{i}),\pi_l(\mathbf{j}))\in\mu}\mathbf{A}_{\mathbf{i}}w_\mu$ holds for all $\mathbf{i}\in[n]^k$, $\mathbf{j}\in[n]^l$, $\pi_k\in S_k$, $\pi_l\in S_l$, and symmetric $\mathbf{A}\in\mathbb{R}^{n^k\times d}$.
Similar to biases, this holds if and only if $w_{\mu_1}=w_{\mu_2}$ for all ($\mu_1,\mu_2$) such that, for some $(\mathbf{i},\mathbf{j})\in{\mu_1}$, $(\mathbf{i}, \pi_l(\mathbf{j}))\in{\mu_2}$ holds for some $\pi_l\in S_l$.
By writing $w_\beta=w_{\mu_1}=w_{\mu_2}$, we can interpret $\beta$ as a new equivalence class of multi-indices formed as a union of all such $\{\mu_1, \mu_2, ...\}$, or equivalently, by collecting all multi-indices related by node and output axis permutations.
On top of that, the symmetry of input tensor gives us another room to reduce the parameters.
We can see that, for all $(\beta_1, \beta_2)$ such that $(\pi_k(\mathbf{i}), \mathbf{j})\in{\beta_2}$ holds for some $(\mathbf{i},\mathbf{j})\in{\beta_1}$ and $\pi_k\in S_k$, we can replace both $w_{\beta_1}$ and $w_{\beta_2}$ by $w_\alpha = (w_{\beta_1} + w_{\beta_2})/2$ and obtain exactly same output.
Extending, for any collection $\{\beta_1, \beta_2, ...\}$ of all such pairwise related $\beta$'s, we can replace $w_{\beta_1}, w_{\beta_2}, ...$ by $w_\alpha = (w_{\beta_1} + w_{\beta_2} + ...)/|\{\beta_1, \beta_2, ...\}|$ and obtain exactly same output.
Thus, we can interpret $\alpha$ as a new equivalence class of multi-indices formed as a union of all such $\{\beta_1, \beta_2, ...\}$, or equivalently, by collecting all multi-indices related by node, input axis, and output axis permutations \emph{i.e.}, equivalence relation $\sim_\alpha$.
Thus, we can reduce weights as $\sum_\alpha\sum_\mathbf{i}\mathbbm{1}_{(\mathbf{i},\mathbf{j})\in\alpha}\mathbf{A}_{\mathbf{i}}w_\alpha$.

\end{proof}

We now prove Proposition~\ref{proposition:equivariant_linear_layer_uniform} and Theorem~\ref{thm:equivariant_linear_layer_hypergraph} (Section~\ref{sec:equivariant_linear_layer_hypergraph}).

\subsubsection{Proof of Proposition~\ref{proposition:equivariant_linear_layer_uniform} (Section~\ref{sec:equivariant_linear_layer_hypergraph})}
We begin from a simple lemma:
\begin{lemma}\label{lemma:index}
For a symmetric order-$k$ tensor $\mathbf{A}^{(k)}$ that represents a $k$-uniform hypergraph (Eq.~\eqref{eqn:hypergraph_tensor}), all indices $\mathbf{i}$ of nonzero entries contain $k$ distinct elements.
\end{lemma}
\begin{proof}
In Eq.~\eqref{eqn:hypergraph_tensor}, recall that $\mathbf{A}^{(k)}_{(i_1, ..., i_k)} \neq 0$ only if $\{i_1, ..., i_k\}\in E^{(k)}$.
As $E^{(k)}$ is a set of $k$-uniform hyperedges that contain $k$ distinct node indices, every multi-indices $(i_1, ..., i_k)$ of nonzero entries of $\mathbf{A}^{(k)}$ contains $k$ distinct elements.
\end{proof}

Now, we prove Proposition~\ref{proposition:equivariant_linear_layer_uniform}.

\begin{proof}
We begin from $L_{[k]\to[l]}$ in Eq.~\eqref{eqn:equivariant_linear_layer_symmetric} and reduce the parameters without affecting the output.
By further constraining the (already symmetric) input and output to symmetric tensors that represent uniform hypergraphs, we first write:
\begin{align}\label{eqn:equivariant_linear_layer_uniform_output_masked}
    L_{(k)\rightarrow (l)}(\mathbf{A}^{(k)})_{\mathbf{j}} = \mathbbm{1}_{|\mathbf{j}|=l}\left(\sum_{\alpha}{\sum_{\mathbf{i}}{\mathbbm{1}_{(\mathbf{i}, \mathbf{j})\in\alpha}\mathbbm{1}_{|\mathbf{i}|=k}\mathbf{A}_{\mathbf{i}}^{(k)}w_{\alpha}}} + \sum_{\beta}{\mathbbm{1}_{\mathbf{j}\in\beta}b_{\beta}}\right).
\end{align}
Note that two constraints are added, $\mathbbm{1}_{|\mathbf{i}|=k}$ multiplied to the input and $\mathbbm{1}_{|\mathbf{j}|=l}$ multiplied to the output.
This comes from the fact that the input/output of the layer are order-$k$/$l$ tensors that represent $k$/$l$-uniform hypergraphs, respectively.
As Lemma~\ref{lemma:index} states, the input/output must contain nonzero entry only for indices that contain $k$/$l$ distinct elements, respectively.

We first reduce biases.
The constraint $\mathbbm{1}_{|\mathbf{j}| = l}$ leaves only a single bias $b_{\beta_l}$ of equivalence class $\beta_l$ and discards the rest (\emph{i.e.}, $|\mathbf{j}| = l$ if and only if $\mathbf{j}\in\beta_l$).
This particular equivalence class $\beta_l$ specifies that all multi-index entries $j_1, ..., j_l$ are unique.
For any other equivalence class $\beta'\neq\beta_l$, the partition that represents it ties at least two entries of $\mathbf{j}\in\beta'$ ($j_a=j_b$ for $a\neq b$).
This gives $|\mathbf{j}|<l$ for all $\mathbf{j}\in\beta'$, which leads to $\mathbf{j}\notin\beta'$ for all $|\mathbf{j}|=l$.
Thus we have $\mathbbm{1}_{|\mathbf{j}|=l}\mathbbm{1}_{\mathbf{j}\in\beta'}=0$ for all $\mathbf{j}$, meaning $b_{\beta'}$ cannot affect the output in Eq.~\eqref{eqn:equivariant_linear_layer_uniform_output_masked}.
We can safely remove all $\beta'\neq\beta_l$ from the bias, and the layer reduces to the following where $b_l = b_{\beta_l}$:
\begin{align}\label{eqn:equivariant_linear_layer_uniform_bias_reduced}
    L_{(k)\rightarrow (l)}(\mathbf{A}^{(k)})_{\mathbf{j}} = \mathbbm{1}_{|\mathbf{j}|=l}\left(\sum_{\alpha}{\sum_{\mathbf{i}}{\mathbbm{1}_{(\mathbf{i}, \mathbf{j})\in\alpha}\mathbbm{1}_{|\mathbf{i}|=k}\mathbf{A}_{\mathbf{i}}^{(k)}w_{\alpha}}} + b_l\right).
\end{align}

We now reduce weights.
Similar to bias, the joint constraint $\mathbbm{1}_{|\mathbf{j}|=l}\mathbbm{1}_{|\mathbf{i}|=k}$ leaves exactly $1 + \text{min}(k, l)$ weights and discards the rest.
We derive this by removing all equivalence classes $\alpha$ that never affect the output.
For any equivalence class $\alpha'$ that ties at least two entries within $\mathbf{i}$ or $\mathbf{j}$ for some $(\mathbf{i}, \mathbf{j})\in\alpha'$ ($i_a=i_b$ or $j_a=j_b$ for $a\neq b$), we have $|\mathbf{i}|<k$ or $|\mathbf{j}|<l$ for all $(\mathbf{i},\mathbf{j})\in\alpha'$, which leads to $(\mathbf{i},\mathbf{j})\notin\alpha'$ for all $|\mathbf{i}|=k$ and $|\mathbf{j}|=l$.
Thus we have  $\mathbbm{1}_{|\mathbf{j}|=l}\mathbbm{1}_{|\mathbf{i}|=k}\mathbbm{1}_{(\mathbf{i}, \mathbf{j})\in\alpha'}=0$ for all $(\mathbf{i},\mathbf{j})$, meaning $w_{\alpha'}$ cannot affect the output in Eq.~\eqref{eqn:equivariant_linear_layer_uniform_output_masked} and can be safely removed.
We now have a reduced set of equivalence classes $\alpha$ such that for any $(\mathbf{i}, \mathbf{j})\in\alpha$, $|\mathbf{i}|=k$ and $|\mathbf{j}|=l$.
Notably, due to axis permutation symmetry described in Lemma~\ref{lemma:equivariant_linear_layer_symmetric}, we can see that each equivalence class $\alpha$ can be described exactly by the number of equivalences $i_a=j_b$ between $\mathbf{i}$ and $\mathbf{j}$, which we denote as $\mathcal{I}$ (\emph{i.e.}, if $\alpha$ is described by $\mathcal{I}$, $|\mathbf{i}\cap\mathbf{j}|=\mathcal{I}$ if and only if $(\mathbf{i}, \mathbf{j})\in\alpha$).
Thus, by discarding irrelevant equivalence classes and rewriting in terms of the overlap $\mathcal{I}$, we can reduce the weight as follows, where $w_\mathcal{I} = w_{\alpha}$ for $\alpha$ that corresponds to $\mathcal{I}$:
\begin{align}\label{eqn:weight_application_reduced_i}
    \sum_{\alpha}{\sum_{\mathbf{i}}{\mathbbm{1}_{(\mathbf{i}, \mathbf{j})\in\alpha}\mathbf{A}_{\mathbf{i}}^{(k)}w_{\alpha}}} =\sum_{\mathcal{I}=0}^{\text{min}(k,l)}\sum_\mathbf{i}\mathbbm{1}_{|\mathbf{j}|=l}\mathbbm{1}_{|\mathbf{i}\cap\mathbf{j}|=\mathcal{I}}\mathbbm{1}_{|\mathbf{i}|=k}\mathbf{A}_{\mathbf{i}}^{(k)}w_\mathcal{I}.
\end{align}

With this, the layer in Eq.~\ref{eqn:equivariant_linear_layer_uniform_bias_reduced} reduces to:
\begin{align}\label{eqn:equivariant_linear_layer_uniform_unreduced}
    L_{(k)\rightarrow (l)}(\mathbf{A}^{(k)})_{\mathbf{j}} = \mathbbm{1}_{|\mathbf{j}|=l}\left(\sum_{\mathcal{I}=0}^{\text{min}(k, l)}{\sum_{\mathbf{i}}{\mathbbm{1}_{|\mathbf{i}\cap\mathbf{j}|=\mathcal{I}}\mathbf{A}_{\mathbf{i}}^{(k)}w_\mathcal{I}}} + b_l\right).
\end{align}
Note that $\mathbbm{1}_{|\mathbf{i}|=k}$ in weight application was removed as $\mathbf{A}^{(k)}$ already contains nonzero entries only for $|\mathbf{i}|=k$ (see Lemma~\ref{lemma:index}).

We finish by handling $\mathcal{I}=0$ in weight application as a special case:
\begin{align}
    \mathbbm{1}_{|\mathbf{j}|=l}&\sum_{\mathcal{I}=0}^{\text{min}(k,l)}\sum_\mathbf{i}\mathbbm{1}_{|\mathbf{i}\cap\mathbf{j}|=\mathcal{I}}\mathbf{A}_{\mathbf{i}}^{(k)}w_\mathcal{I} \nonumber\\
    &=\mathbbm{1}_{|\mathbf{j}|=l}\left(
    \sum_{\mathcal{I}=1}^{\text{min}(k,l)}\sum_\mathbf{i}\mathbbm{1}_{|\mathbf{i}\cap\mathbf{j}|=\mathcal{I}}\mathbf{A}_{\mathbf{i}}^{(k)}w_\mathcal{I}
    + \sum_\mathbf{i}\mathbbm{1}_{|\mathbf{i}\cap\mathbf{j}|=0}\mathbf{A}_{\mathbf{i}}^{(k)}w_0
    \right)\\
    &=\mathbbm{1}_{|\mathbf{j}|=l}\left(
    \sum_{\mathcal{I}=1}^{\text{min}(k,l)}\sum_\mathbf{i}\mathbbm{1}_{|\mathbf{i}\cap\mathbf{j}|=\mathcal{I}}\mathbf{A}_{\mathbf{i}}^{(k)}(w_\mathcal{I}-w_0)
    + \sum_\mathbf{i}\mathbbm{1}_{|\mathbf{i}\cap\mathbf{j}|\geq0}\mathbf{A}_{\mathbf{i}}^{(k)}w_0
    \right)\\
    &=\mathbbm{1}_{|\mathbf{j}|=l}\left(
    \sum_{\mathcal{I}=1}^{\text{min}(k,l)}\sum_\mathbf{i}\mathbbm{1}_{|\mathbf{i}\cap\mathbf{j}|=\mathcal{I}}\mathbf{A}_{\mathbf{i}}^{(k)}w'_\mathcal{I}
    + \sum_\mathbf{i}\mathbf{A}_{\mathbf{i}}^{(k)}w_0
    \right).
\end{align}
This modification is not strictly required, but makes practical implementation much easier as $\mathcal{I}=0$ case reduces to global pooling.
By rewriting $w'_{\mathcal{I}}$ as $w_{\mathcal{I}}$, we finally arrive at the reduced layer in Eq.~\eqref{eqn:equivariant_linear_layer_uniform}.
\end{proof}

\subsubsection{Proof of Theorem~\ref{thm:equivariant_linear_layer_hypergraph} (Section~\ref{sec:equivariant_linear_layer_hypergraph})}

As a direct proof upon tensor sequences is not straightforward, we first \emph{pack} the entire sequence of tensors $\mathbf{A}^{(:K)}$ that represent a hypergraph into an equivalent symmetric order-$K$ tensor $\mathbf{A}^{[K]}\in\mathbb{R}^{n^K\times d}$.
Then, we show that the maximally expressive equivariant linear layer $L_{[K]\to[L]}:\mathbb{R}^{n^K\times d}\to\mathbb{R}^{n^L\times d}$ for the packed tensors is equivalent to $L_{(:K)\to(:L)}$ for tensor sequences, and thus they are maximally expressive.

We first characterize the packed tensors with the following lemma:

\begin{lemma}\label{lemma:pack}
A sequence of tensors $\mathbf{A}^{(:K)}=\left(\mathbf{A}^{(k)}\right)_{k\leq K}$ that represent a hypergraph (Definition~\ref{defn:hypergraph_tensor_sequence}) can be represented as an equivalent symmetric order-$K$ tensor $\mathbf{A}^{[K]}\in\mathbb{R}^{n^K\times d}$.
\end{lemma}
\begin{proof}
Let us define $\mathbf{A}^{[K]}\in\mathbb{R}^{n^K\times d}$ as follows:
\begin{align}\label{eqn:hypergraph_sequence_tensor}
    \mathbf{A}^{[K]}_{(i_1, ..., i_K)} = \mathbf{A}^{(|\{i_1, ..., i_K\}|)}_{\{i_1, ..., i_K\}},
\end{align}
where $\{i_1, ..., i_K\}$ denotes the set of unique entries in $(i_1, ..., i_K)$ ordered arbitrarily (note that arbitrary ordering always gives the same value as $\mathbf{A}^{(k)}$ is symmetric).
From $\mathbf{A}^{[K]}$, we can retrieve the original sequence of tensors $\left(\mathbf{A}^{(k)}\right)_{k\leq K}$ by using all order-$K$ multi-indices that contain $k$ unique entries to index $\mathbf{A}^{[K]}$ to construct each $\mathbf{A}^{(k)}$.
A convenient property of $\mathbf{A}^{[K]}$ is that, for any pair of multi-indices $\mathbf{i}, \mathbf{j}\in[n]^K$ that contain the same number of unique elements, $\mathbf{A}^{[K]}_\mathbf{i} = \mathbf{A}^{[K]}_\mathbf{j}$ holds.
\end{proof}

Now, we prove Theorem~\ref{thm:equivariant_linear_layer_hypergraph}.
The proof is analogous to Lemma~\ref{lemma:equivariant_linear_layer_symmetric} as we reduce the weights and biases by leveraging the constraint on input and output (Eq.~\eqref{eqn:hypergraph_sequence_tensor}).

\begin{proof}
We prove by showing the equivalence of $L_{[K]\to[L]}$ (Eq.~\eqref{eqn:equivariant_linear_layer_symmetric}) to the collection $\left(L_{(k)\to(l)}\right)_{k\leq K, l\leq L}$ (Eq.~\eqref{eqn:equivariant_linear_layer_uniform_unreduced}).
We begin by writing $L_{[K]\to[L]}$ as a maximally expressive equivariant linear layer for symmetric tensors (Eq.~\eqref{eqn:equivariant_linear_layer_symmetric}):
\begin{align}\label{eqn:packed_equivariant_linear_layer}
    L_{[K]\to [L]}(\mathbf{A}^{[K]})_\mathbf{j} = \sum_\alpha\sum_\mathbf{i}\mathbbm{1}_{(\mathbf{i},\mathbf{j})\in\alpha}\mathbf{A}^{[K]}_\mathbf{i}w_\alpha + \sum_\beta\mathbbm{1}_{\mathbf{j}\in\beta}b_\beta,
\end{align}

Let us first reduce the biases.
By definition, the bias $\sum_\beta\mathbbm{1}_{\mathbf{j}\in\beta}b_\beta$ (Eq.~\eqref{eqn:packed_equivariant_linear_layer}) is constrained to be an instantiation of the tensor described in Eq.~\eqref{eqn:hypergraph_sequence_tensor}.
Due to the property of such tensors, bias entries at all multi-indices having a specific number of unique elements are constrained to have an identical value.
This holds if and only if $b_{\beta_1}=b_{\beta_2}$ for all $(\beta_1, \beta_2)$ such that, any $\mathbf{j}_1\in\beta_1$ and $\mathbf{j}_2\in\beta_2$ have a same number (say, $l$) of unique elements.
By writing $b_{\beta_l}=b_{\beta_1}=b_{\beta_2}$, $\beta_l$ can be interpreted as a new equivalence class formed as a union of all such $\{\beta_1,\beta_2, ...\}$, which is equivalently the collection of all multi-indices with $l$ unique elements.
With this, the biases in Eq.~\eqref{eqn:equivariant_linear_layer_symmetric} can be rewritten as follows:
\begin{align}
    \sum_\beta\mathbbm{1}_{\mathbf{j}\in\beta}b_\beta = \sum_{l\leq L}\mathbbm{1}_{|\mathbf{j}|=l}b_{\beta_l},
\end{align}
which is equivalent to the collection of biases in $\left(L_{(k)\to(l)}\right)_{k\leq K, l\leq L}$ (Eq.~\eqref{eqn:equivariant_linear_layer_uniform_unreduced}).

We now move to the weight.
By definition, the input and output of $L_{[K]\to[L]}$ are constrained to be tensors described in Eq.~\eqref{eqn:hypergraph_sequence_tensor}.
As a result, input entries at all multi-indices with a specific number of unique elements have an identical value; and the same constrained is applied for output as well.
Similar to the bias, this holds if and only if $w_{\alpha_1}=w_{\alpha_2}$ for all ($\alpha_1, \alpha_2$) such that, for some $(\mathbf{i},\mathbf{j}_1)\in{\alpha_1}$, $(\mathbf{i}, \mathbf{j}_2)\in{\alpha_2}$ where $\mathbf{j}_1$ and $\mathbf{j}_2$ have an identical number ($l$) of unique elements.
By writing $w_{\alpha'}=w_{\alpha_1}=w_{\alpha_2}$, $\alpha'$ is interpreted as a new equivalence class formed as a union of all such $\{\alpha_1, \alpha_2, ...\}$.
In addition, we can leverage the structure of the input to further reduce the weight parameters.
For all $(\alpha'_1, \alpha'_2)$ such that, for some $(\mathbf{i}_1,\mathbf{j})\in{\alpha'_1}$, $(\mathbf{i}_2, \mathbf{j})\in{\alpha'_2}$ where $\mathbf{i}_1$ and $\mathbf{i}_2$ have an identical number ($k$) of unique elements and identical number ($\mathcal{I}$) of overlapping unique elements to $\mathbf{j}$, we can replace both $w_{\alpha'_1}$ and $w_{\alpha'_2}$ by $w_{\alpha_{k,l,\mathcal{I}}} = (w_{\alpha'_1} + w_{\alpha'_2})/2$ and obtain exactly same output.
Extending this, for any collection $\{\alpha'_1, \alpha'_2, ...\}$ of all such pairwise related~$\alpha'$, we can replace $w_{\alpha'_1}, w_{\alpha'_2}, ...$ by $w_{\alpha_{k,l,\mathcal{I}}} = (w_{\alpha'_1} + w_{\alpha'_2} + ...)/|\{\alpha'_1, \alpha'_2, ...\}|$ and still obtain the exact same output.
$\alpha_{k,l,\mathcal{I}}$ is thus interpreted as a new equivalence class formed as a union of all such $\{\alpha'_1, \alpha'_2, ...\}$, or equivalently, formed as a collection of all multi-indices $(\mathbf{i}, \mathbf{j})$ where $\mathbf{i}\in[n]^K,\mathbf{j}\in[n]^L$ with $k$ unique entries in $\mathbf{i}$, $l$ unique entries in $\mathbf{j}$, and $\mathcal{I}$ tying relations $i_a=i_b=...=j_c=j_d$.
With this, the weights in Eq.~\eqref{eqn:equivariant_linear_layer_symmetric} can be rewritten as follows:
\begin{align}
    \sum_\alpha\sum_\mathbf{i}\mathbbm{1}_{(\mathbf{i},\mathbf{j})\in\alpha}\mathbf{A}^{[K]}_\mathbf{i}w_\alpha = \sum_{k\leq K}\sum_{l\leq L}\sum_{\mathcal{I}=0}^{\textnormal{min}(K, L)}\sum_\mathbf{i}\mathbbm{1}_{|\mathbf{j}|=l}\mathbbm{1}_{|\mathbf{i}\cap\mathbf{j}|=\mathcal{I}}\mathbbm{1}_{|\mathbf{i}|=k}\mathbf{A}^{[K]}_\mathbf{i}w_{\alpha_{k,l,\mathcal{I}}},
\end{align}
which is computationally equivalent to the collection of weight applications in $\left(L_{(k)\to(l)}\right)_{k\leq K, l\leq L}$.

\end{proof}

\subsubsection{Proof of Theorem~\ref{thm:ehnn_mlp}~(Section~\ref{sec:ehnn_realization})}
\begin{proof}
Our target of approximation is the following, the output of weight application from Eq.~\eqref{eqn:ehnn}~(Section~\ref{sec:ehnn}):
\begin{align}
    w(\mathbf{A}^{(:K)})_{l, \mathbf{j}} &= \sum_{\mathcal{I}=0}^K\sum_{k\leq K} \sum_{\mathbf{i}}\textbf{B}_{\mathbf{i},\mathbf{j}}^{\mathcal{I}} \mathbf{A}^{(k)}_{\mathbf{i}}\mathcal{W}(k,l,\mathcal{I}),\\
    \text{where }\mathbf{B}_{\mathbf{i},\mathbf{j}}^\mathcal{I} &= 
    \left\{\begin{array}{cc}
        \mathbbm{1}_{|\mathbf{i}\cap\mathbf{j}|=\mathcal{I}} & \scalebox{0.95}{if $\mathcal{I}\geq 1$} \\
        1 & \scalebox{0.95}{if $\mathcal{I} = 0$}
    \end{array}\right..
\end{align}

Note that we moved the summation $\sum_{\mathcal{I}=1}^{\textnormal{min}(k, l)}$ out of $\sum_{k\leq K}$, and adjusted the summation range to $\sum_{\mathcal{I}=1}^K$.
This does not change the output because for any $\mathcal{I}\geq\text{min}(k, l)$ we have $\mathbbm{1}_{|\mathbf{i}\cap\mathbf{j}|=\mathcal{I}}=0$.

We now introduce MLPs $\phi_1:\mathbb{N}\times\mathbb{R}^{d}\to\mathbb{R}^{Kd}$, $\phi_2:\mathbb{N}\times\mathbb{R}^{Kd}\to\mathbb{R}^{(1+K)Kd}$, and $\phi_3:\mathbb{N}\times\mathbb{R}^{(1+K)Kd}\to\mathbb{R}^{d}$ to model appropriate functions based on universal approximation~\cite{hornik1989multilayer}.

First, we make $\phi_1$ output the following, where $\epsilon_1$ is approximation error:
\begin{align}\label{eqn:phi_1}
    \phi_1(k, \mathbf{X})_{(k-1)d+1:kd} = \mathbf{X} + \epsilon_1,
\end{align}
so that $\phi_1(0, \mathbf{X})$ places $\mathbf{X}$ on the first $d$ channels of the output, $\phi_1(1, \mathbf{X})$ places $\mathbf{X}$ on the $d+1:2d$ channels of the output, and so on.
Then, $\sum_{k\leq K}\sum_{\mathbf{i}}\textbf{B}_{\mathbf{i},\mathbf{j}}^{\mathcal{I}}\phi_1(k, \mathbf{A}^{(k)}_{\mathbf{i}})$ gives channel concatenation of $\left(\sum_{\mathbf{i}}\textbf{B}_{\mathbf{i},\mathbf{j}}^{\mathcal{I}}\mathbf{A}^{(k)}_{\mathbf{i}}\right)_{k\leq K}$, which we denote $\mathbf{Y}_\mathcal{I}$.

Then, we make $\phi_2$ output the following, where $\epsilon_2$ is approximation error:
\begin{align}\label{eqn:phi_2}
    \phi_2(\mathcal{I}, \mathbf{Y})_{\mathcal{I}Kd+1:(\mathcal{I}+1)Kd} = \mathbf{Y} + \epsilon_2.
\end{align}

Then, $\sum_{\mathcal{I}=0}^K\phi_2(\mathcal{I}, \mathbf{Y}_\mathcal{I})$ gives a concatenation of $\mathbf{Y}_{0:K}$, which is equivalently concatenation of $\left(\left(\sum_{\mathbf{i}}\textbf{B}_{\mathbf{i},\mathbf{j}}^{\mathcal{I}}\mathbf{A}^{(k)}_{\mathbf{i}}\right)_{k\leq K}\right)_{\mathcal{I}\leq K}$.

Finally, we let $\phi_3$ output the following, where $\epsilon_3$ is approximation error:
\begin{align}\label{eqn:phi_3}
    \phi_3(l, \mathbf{Z}) = \sum_{\mathcal{I}=0}^{K}\sum_{k\leq K}\mathbf{Z}_{(\mathcal{I}K + k - 1)d+1:(\mathcal{I}K + k)d}\mathcal{W}(k, l, \mathcal{I})+ \epsilon_3,
\end{align}
which, by indexing back through $\mathbf{Z}$, $\mathbf{Y}$, and $\mathbf{X}$, gives $w(\mathbf{A}^{(:K)})_{l, \mathbf{j}}$.

We finish by rewriting $\phi_3(l, \mathbf{Z})$ as follows:
\begin{align}
    \phi_3(l, \mathbf{Z}) = \phi_3\left(l,
    \sum_{\mathcal{I}\geq 0}
    \phi_2\left(\mathcal{I},
    \sum_{k\leq K} \sum_{\mathbf{i}}\textbf{B}_{\mathbf{i},\mathbf{j}}^{\mathcal{I}}
    \phi_1(k, \mathbf{A}^{(k)}_{\mathbf{i}})
    \right)
    \right),
\end{align}
which is equivalent to the output of $\phi_3$ of EHNN-MLP in Eq.~\eqref{eqn:ehnn_mlp}.
Therefore, with $\phi_1$, $\phi_2$, and $\phi_3$ approximating the functions in Eq.~\eqref{eqn:phi_1}, Eq.~\eqref{eqn:phi_2}, and Eq.~\eqref{eqn:phi_3} respectively, the output of $\phi_3$ of EHNN-MLP can approximate the output of weight application of EHNN~(Eq.~\eqref{eqn:ehnn}) to arbitrary precision\footnote{The overall error depend on the approximation error of MLPs ($\phi_1,\phi_2,\phi_3$) at each step, and uniform bounds and continuity of the modeled functions \cite{hornik1989multilayer}.}.
\end{proof}

\subsubsection{Proof of Theorem~\ref{thm:alldeepsets} (Section~\ref{sec:ehnn_realization})}
\begin{proof}
We can reduce an EHNN-MLP layer to an AllDeepSets layer with following procedure.
First, from Eq.~\eqref{eqn:ehnn_mlp}, we let $\phi_1(k, \mathbf{X}) = \phi_1'(\mathbf{X})$, $\phi_3(l, \mathbf{X}) = \phi_3'(\mathbf{X})$, and $\mathcal{B}(l)=0$ to remove conditioning on hyperedge orders $k$ and $l$.
Then we let $\phi_2(\mathcal{I}, \mathbf{X})=\mathbbm{1}_{\mathcal{I}}\mathbf{X}$ to ablate the global interaction ($\mathcal{I}=0$) and remove conditioning on $\mathcal{I}\geq 1$.
By renaming $\phi_1'$ to $\phi_1$ and $\phi_3'$ to $\phi_2$, we get the AllDeepSets layer (Eq.~\eqref{eqn:alldeepsets}).
Yet, an AllDeepSets layer cannot reduce to an EHNN-MLP layer as it cannot model interactions between non-overlapping hyperedges ($\mathcal{I}=0$).
\end{proof}

\subsection{In-Depth Comparison of EHNN and AllSet (Section~\ref{sec:ehnn_realization})}\label{sec:apdx_in_depth_comparison_ehnn_allset}

Note that in the proof of Theorem~\ref{thm:alldeepsets}, we leveraged the simple property that message passing of AllDeepSets cannot account for long-range dependency modeling of EHNN.
Then a natural question arises here: if we rule out the presence of global interaction, can the \emph{local message passing} of EHNN be potentially better than that of current message passing networks?
Our analysis suggests so.
More specifically, in practical scenarios when local message aggregation cannot be multiset universal, the explicit hyperedge order conditioning of EHNN might help.
To demonstrate this, we bring a powerful, Transformer-based characterization of AllSet framework called AllSetTransformer:
\begin{align}
    \text{AllSetAttn}(\mathbf{A}^{(:K)})_{l, \mathbf{j}} &=
    \mathbbm{1}_{|\mathbf{j}|=l}\sum_{h=1}^H\sum_{k\leq K} \sum_{\mathbf{i}}\boldsymbol{\alpha}_{\mathbf{i},\mathbf{j}}^{h}
    \phi_1(\mathbf{A}^{(k)}_{\mathbf{i}}),\label{eqn:allsettransformer_attn}\\
    \text{AllSetTransformer}(\mathbf{A}^{(:K)}) &= \text{AllSetAttn}(\mathbf{A}^{(:K)}) + \text{MLP}(\text{AllSetAttn}(\mathbf{A}^{(:K)})),\label{eqn:allsettransformer}
\end{align}
with $H$ the number of heads, and attention coefficients $\boldsymbol{\alpha}_{\mathbf{i},\mathbf{j}}^{h}$ computed with a query matrix $Q\in\mathbb{R}^{H\times d_H}$, a key network $\mathcal{K}:\mathbb{R}^d\to\mathbb{R}^{H\times d_H}$, and activation $\sigma(\cdot)$:
\begin{align}\label{eqn:allsettransformer_attcoef}
    \boldsymbol{\alpha}_{\mathbf{i},\mathbf{j}}^{h} &=
    \sigma\left(Q_h
    \mathcal{K}\left(\mathbf{A}^{(k)}_{\mathbf{i}}\right)_h^\top
    /\sqrt{d_H}\cdot\mathbbm{1}_{|\mathbf{i}\cap\mathbf{j}|=\mathcal{I}}
    \right).
\end{align}

An interesting fact here is that when we take the activation for attention $\sigma(\cdot)$ as \emph{any} kinds of normalization including softmax, attention mechanism of AllSetTransformer can fail to model even trivial multiset functions.
A simple example is that they cannot \emph{count} the number of input hyperedges that overlap with each output hyperedge.
As a simple demonstration, if input hypergraph is a node set $\mathbf{A}^{(1)}\in\mathbb{R}^{n\times d}$ with identical features $\mathbf{A}_i^{(1)}=\boldsymbol{1}$, an AllSetTransformer layer that outputs features of hyperedges $E$ cannot learn node counting \emph{i.e.}, for output $\mathbf{Y}$, $\mathbf{Y}_e = |e|$ for $e\in E$.
This is because, as attention normalizes over keys, all outputs of AllSetAttn (Eq.~\eqref{eqn:allsettransformer_attn}) is always $H\phi_1(\boldsymbol{1})$.
On the other hand, the hyperedge order-aware message passing of EHNN-Transformer can easily solve the problem by forwarding $l$.

Let us finish with stronger expressiveness of EHNN-Transformer (Eq.~\eqref{eqn:ehnn_transformer}) compared to AllSetTransformer:
\begin{theorem}\label{thm:allsettransformer}
An AllSetTransformer layer (Eq.~\eqref{eqn:allsettransformer}) is a special case of EHNN-Transformer layer (Eq.~\eqref{eqn:ehnn_transformer}), while the opposite is not true.
\end{theorem}
\begin{proof}
We can reduce an EHNN-Transformer layer to an AllSetTransformer layer with following procedure.
First, from Eq.~\eqref{eqn:ehnn_transformer_attn}, we let $\phi_1(k, \mathbf{X}) = \phi_1'(\mathbf{X})$, $\phi_3(l, \mathbf{X}) = \mathbf{X}$, and $\mathcal{B}(l)=0$ to remove conditioning on hyperedge orders $k$ and $l$.
Then we let $\phi_2(\mathcal{I}, \mathbf{X})=\mathbbm{1}_{\mathcal{I}}\mathbf{X}$ to ablate the global interaction ($\mathcal{I}=0$) and remove conditioning on $\mathcal{I}\geq 1$.
Lastly, from Eq.~\eqref{eqn:ehnn_transformer_attcoef}, we let $\mathcal{Q}(\mathcal{I}) = Q$ and $\mathcal{K}(\mathcal{I}, \mathbf{X}) = \mathcal{K}(\mathbf{X})$.
By renaming $\phi_1'$ to $\phi_1$, we get the AllSetTransformer layer (Eq.~\eqref{eqn:allsettransformer}).
Yet, an AllSetTransformer cannot reduce to an EHNN-Transformer as it cannot model interactions between non-overlapping hyperedges ($\mathcal{I}=0$).
\end{proof}

\subsection{Experimental Details (Section~\ref{sec:experiments})}\label{sec:apdx_experimental_details}
We provide detailed experimental details including dataset statistics and hyperparameter search procedure.
We provide dataset statistics in Table~\ref{table:dataset_statistics} and optimal hyperparameter settings in Table~\ref{table:optimal_hyperparameter}.

\begin{table}[!t]
\caption{Statistics of the datasets.}
\label{table:dataset_statistics}
\begin{subtable}[h]{\linewidth}
\caption{Statistics of $k$-edge identification dataset.}
\centering
\label{table:k-edge_identification}
\begin{adjustbox}{width=0.8\textwidth}
        \begin{tabular}{ccccccccc}
        \Xhline{2\arrayrulewidth}\\[-1em]
        \\[-1em] \multirow{2}{5em}{Dataset} & Test involves only seen $k$ & \multicolumn{2}{c}{Test involves unseen $k$} \\
        \\[-1em] & & Interpolation & Extrapolation \\
        \\[-1em]\Xhline{2\arrayrulewidth}\\[-1em]
        \\[-1em] \# nodes & 100 & 100 & 100 \\
        \\[-1em] \# edges & 10 & 10 & 10 \\
        \\[-1em] train edge orders & 2-10 & 2-4, 8-10 & 2-7 \\
        \\[-1em] test edge orders & 2-10 & 2-10 & 2-10 \\
        \\[-1em]\Xhline{2\arrayrulewidth}
    \end{tabular}
\end{adjustbox}
\end{subtable}
\begin{subtable}[h]{\linewidth}
\caption{Statistics of node classification dataset.}
\centering
\label{table:node_classification}
\begin{adjustbox}{width=0.8\textwidth}
        \begin{tabular}{ccccccccc}
        \Xhline{2\arrayrulewidth}\\[-1em]
        \\[-1em] Dataset & Zoo & 20Newsgroups & Mushroom & NTU2012 & ModelNet40 & Yelp & House & Walmart \\
        \\[-1em]\Xhline{2\arrayrulewidth}\\[-1em]
        \\[-1em] \# nodes & 101 & 16242 & 8124 & 2012 & 12311 & 50758 & 1290 & 88860 \\
        \\[-1em] \# edges & 43 & 100 & 298 & 2012 & 12311 & 679302 & 341 & 69906 \\
        \\[-1em] \# feature & 16 & 100 & 22 & 100 & 100 & 1862 & 100 & 100\\
        \\[-1em] \# class & 7 & 4 & 2 & 67 & 40 & 9 & 2 & 11 \\
        \\[-1em]\Xhline{2\arrayrulewidth}
    \end{tabular}
\end{adjustbox}
\end{subtable}
\begin{subtable}[h]{\linewidth}
\caption{Statistics of visual keypoint matching dataset.}
\centering
\label{table:keypoint_matching}
\begin{adjustbox}{width=0.45\textwidth}
        \begin{tabular}{ccc}
        \Xhline{2\arrayrulewidth}\\[-1em]
        \\[-1em] Dataset & Willow & PASCAL-VOC\\
        \\[-1em]\Xhline{2\arrayrulewidth}\\[-1em]
        \\[-1em] \# categories & 5 & 20 \\
        \\[-1em] \# images & 256 & 11,530 \\
        \\[-1em] \# keypoints/image & >10 & 6-23\\
        \\[-1em]\Xhline{2\arrayrulewidth}
    \end{tabular}
\end{adjustbox}
\end{subtable}
\end{table}

\begin{table}[!t]
\caption{Optimal hyperparameters for each dataset.
$lr, wd, d, h$ each refers to learning rate, weight decay, hidden dimension, and number of heads.
$do_{mlp}$ refers to dropout rate on MLP and $do_{local}$/$do_{global}$ refers to dropout rate on local/global interactions.
$d_c$ refers to classifier hidden dimension.}
\label{table:optimal_hyperparameter}
\begin{subtable}[h]{\textwidth}
\caption{Optimal hyperparameter for node classification datasets.}
\centering
\label{table:node_classification_hyperparameter}
\begin{adjustbox}{width=0.8\textwidth}
        \begin{tabular}{ccccccccccccccc}
        \Xhline{2\arrayrulewidth}\\[-1em]
         & & \multicolumn{4}{c}{EHNN-MLP} & & \multicolumn{7}{c}{EHNN-Transformer} \\
        \Xhline{2\arrayrulewidth}\\[-1em]
        \\[-1em] & \hspace{0.5cm} & $lr$ & $wd$ & $d$ & $do_{mlp}$ & \hspace{0.5cm} & $lr$ & $wd$ & $h$ & $d$ & $do_{mlp}$ & $do_{local}$ & ${do_{global}}$ & $d_c$ \\
        \\[-1em]\Xhline{2\arrayrulewidth}\\[-1em]
        \\[-1em] Zoo & & 0.001 & 0 & 128 & 0 & & 0.001 & 0 & 8 & 256 & 0 & 0 & 0 & 128 \\
        \\[-1em] 20Newsgroups & & 0.001 & 1e$^{-5}$ & 256 & 0 & & 0.001 & 1e$^{-5}$ & 8 & 256 & 0 & 0 & 0 & 64\\
        \\[-1em] Mushroom & & 0.001 & 1e$^{-5}$ & 128 & 0 & & 0.001 & 1e$^{-5}$ & 4 & 256 & 0 & 0 & 0 & 128\\
        \\[-1em] NTU2012 & & 0.001 & 0 & 256 & 0.2 & & 0.001 & 1e$^{-5}$ & 8 & 256 & 0.2 & 0.1 & 0.1 & 64\\
        \\[-1em] ModelNet40 & & 0.001 & 1e$^{-5}$ & 256 & 0.2 & & 0.001 & 1e$^{-5}$ & 4 & 256 & 0 & 0 & 0 & 64\\
        \\[-1em] Yelp & & 0.001 & 0 & 64 & 0 & & 0.001 & 0 & 8 & 64 & 0 & 0 & 0 & 128\\
        \\[-1em] House(1) & & 0.001 & 0 & 256 & 0.2 & & 0.001 & 1e$^{-5}$ & 4 & 64 & 0.2 & 0 & 0 & 64\\
        \\[-1em] Walmart(1) & & 0.001 & 0 & 256 & 0.2 & & 0.001 & 0 & 8 & 256 & 0.2 & 0 & 0 & 64\\
        \\[-1em] House(0.6) & & 0.001 & 0 & 128 & 0 & & 0.001 & 0 & 4 & 256 & 0 & 0 & 0 & 64\\
        \\[-1em] Walmart(0.6) & & 0.001 & 1e$^{-5}$ & 256 & 0.2 & & 0.001 & 1e$^{-5}$ & 8 & 128 & 0.2 & 0 & 0 & 64\\
        \\[-1em]\Xhline{2\arrayrulewidth}
    \end{tabular}
\end{adjustbox}
\end{subtable}
\begin{subtable}[h]{\textwidth}
\caption{Optimal hyperparameter for visual keypoint matching datasets.}
\centering
\label{table:keypoint_matching_hyperparameter}
\begin{adjustbox}{width=0.8\textwidth}
        \begin{tabular}{cccccccccccccc}
        \Xhline{2\arrayrulewidth}\\[-1em]
         & & \multicolumn{4}{c}{EHNN-MLP} & & \multicolumn{7}{c}{EHNN-Transformer} \\
        \Xhline{2\arrayrulewidth}\\[-1em]
        \\[-1em] & \hspace{0.5cm} & $lr$ & $wd$ & $d$ & $do_{mlp}$ & \hspace{0.5cm} & $lr$ & $wd$ & $h$ & $d$ & $do_{mlp}$ & $do_{local}$ & $do_{global}$\\
        \\[-1em]\Xhline{2\arrayrulewidth}\\[-1em]
        \\[-1em] Willow & & 2e$^{-4}$ & 0.5 & 32 & 0 & & 2e$^{-4}$ & 0.5 & 4 & 32 & 0.1 & 0 & 0.5\\
        \\[-1em] PASCAL-VOC & & 2e$^{-4}$ & 0.5 & 32 & 0 & & 2e$^{-4}$ & 0.5 & 4 & 128 & 0 & 0 & 0\\
        \\[-1em]\Xhline{2\arrayrulewidth}
    \end{tabular}
\end{adjustbox}
\end{subtable}
\end{table}

\subsubsection{Synthetic $k$-edge Identification}
For synthetic $k$-edge identification, we used small datasets composed of 100 train and 20 test hypergraphs, each with 100 nodes and randomly wired 10 hyperedges.
In input hypergraph, we pick a random hyperedge and mark its nodes with a binary label.
The task is to identify (classify) every other nodes whose hyperedge order is the same with the marked one.
For default training set, we sample hyperedges of orders $\in\{2, ..., 10\}$.
To further test generalization of models to unseen orders, we use two additional training sets where hyperedges are sampled \emph{without} order-$\{5, 6, 7\}$ hyperedges (interpolation) or order-$\{8, 9, 10\}$ hyperedges (extrapolation) as in Table~\ref{table:k-edge_identification}.
For all models, we use simple two-layer architecture that first converts the node labels to hyperedge features ($V\to E$) then maps them back to nodes ($E\to V$) for classification.
For all models we set hidden dimension to 64, and for EHNN-Transformer and AllSetTransformer we set number of attention heads to 4.

\subsubsection{Semi-supervised Classification}
\begin{table}[t!]
\begin{center}
\caption{
Results for semi-supervised node classification.
Average accuracy (\%) over 20 runs are shown with standard deviation.
}
\label{table:semi_supervised_with_std}
\resizebox{\textwidth}{!}{
\begin{tabular}{cccccccccccccccc}
\toprule
 & Zoo & 20Newsgroups & mushroom & NTU2012 & ModelNet40 & Yelp & House(1) & Walmart(1) & House(0.6) & Walmart(0.6) & avg. rank ($\uparrow$)\\
\midrule
MLP & 87.18 ± 4.44 & \cellcolor{gray!60}\textbf{81.42 ± 0.49} & \cellcolor{gray!60}\textbf{100.00 ± 0.00} & 85.52 ± 1.49 & 96.14 ± 0.36 & 31.96 ± 0.44 & 67.93 ± 2.33 & 45.51 ± 0.24 & 81.53 ± 2.26 & 63.28 ± 0.37 & 6.4 \\
CEGCN & 51.54 ± 11.19 & OOM & 95.27 ± 0.47 & 81.52 ± 1.43 & 89.92 ± 0.46 & OOM & 62.80 ± 2.61 & 54.44 ± 0.24 & 64.36 ± 2.41 & 59.78 ± 0.32 & 11.5 \\
CEGAT & 47.88 ± 14.03 & OOM & 96.60 ± 1.67 & 82.21 ± 1.23 & 92.52 ± 0.39 & OOM & \cellcolor{blue!30}69.09 ± 3.00 & 51.14 ± 0.56 & 77.25 ± 2.53 & 59.47 ± 1.05 & 10.5 \\
HNHN & 93.59 ± 5.88 & \cellcolor{blue!30}81.35 ± 0.61 & \cellcolor{gray!60}\textbf{100.00 ± 0.01} & \cellcolor{blue!30}89.11 ± 1.44 & 97.84 ± 0.25 & 31.65 ± 0.44 & 67.80 ± 2.59 & 47.18 ± 0.35 & 78.78 ± 1.88 & 65.80 ± 0.39 & 5.9 \\
HGNN & 92.50 ± 4.58 & 80.33 ± 0.42 & 98.73 ± 0.32 & 87.72 ± 1.35 & 95.44 ± 0.33 & 33.04 ± 0.62 & 61.39 ± 2.96 & 62.00 ± 0.24 & 66.16 ± 1.80 & 77.72 ± 0.21 & 7.8 \\
HCHA & 93.65 ± 6.15 & 80.33 ± 0.80 & 98.70 ± 0.39 & 87.48 ± 1.87 & 94.48 ± 0.28 & 30.99 ± 0.72 & 61.36 ± 2.53 & 62.45 ± 0.26 & 67.91 ± 2.26 & 77.12 ± 0.26 & 8.1 \\
HyperGCN & N/A & \cellcolor{blue!30}81.05 ± 0.59 & 47.90 ± 1.04 & 56.36 ± 4.86 & 75.89 ± 5.26 & 29.42 ± 1.54 & 48.31 ± 2.93 & 44.74 ± 2.81 & 78.22 ± 2.46 & 55.31 ± 0.30 & 12.4 \\
UniGCNII & 93.65 ± 4.37 & \cellcolor{blue!30}81.12 ± 0.67 & 99.96 ± 0.05 & \cellcolor{blue!30}89.30 ± 1.33 & 98.07 ± 0.23 & 31.70 ± 0.52 & 67.25 ± 2.57 & 54.45 ± 0.37 & 80.65 ± 1.96 & 72.08 ± 0.28 & 5.8 \\
HAN (full batch) & 85.19 ± 8.18 & OOM & 90.86 ± 2.40 & 83.58 ± 1.46 & 94.04 ± 0.41 & OOM & \cellcolor{blue!30}71.05 ± 2.26 & OOM & \cellcolor{blue!30}83.27 ± 1.62 & OOM & 9.9\\
HAN (minibatch) & 75.77 ± 7.10 & 79.72 ± 0.62 & 93.45 ± 1.31 & 80.77 ± 2.36 & 91.52 ± 0.96 & 26.05 ± 1.37 & 62.00 ± 9.06 & 48.57 ± 1.04 & 82.04 ± 2.68 & 63.1 ± 0.96 & 10.6\\
\midrule
AllDeepSets & \cellcolor{blue!30}95.39 ± 4.77 & \cellcolor{blue!30}81.06 ± 0.54 & 99.99 ± 0.02 & 88.09 ± 1.52 & 96.98 ± 0.26 & 30.36 ± 1.57 & 67.82 ± 2.40 & 64.55 ± 0.33 & 80.70 ± 1.59 & 78.46 ± 0.26 & 5.4\\
AllSetTransformer & \cellcolor{gray!60}\textbf{97.50 ± 3.59} & \cellcolor{blue!30}81.38 ± 0.58 & \cellcolor{gray!60}\textbf{100.00 ± 0.00} & \cellcolor{blue!30}88.69 ± 1.24 & \cellcolor{blue!30}98.20 ± 0.20 & \cellcolor{gray!60}\textbf{36.89 ± 0.51} & \cellcolor{blue!30}69.33 ± 2.20 & 65.46 ± 0.25 & \cellcolor{blue!30}83.14 ± 1.92 & 78.46 ± 0.26 & 2.4\\
\midrule
EHNN-MLP & 91.15 ± 6.13 & \cellcolor{blue!30}81.31 ± 0.43 & 99.99 ± 0.03 & 87.35 ± 1.42 & 97.74 ± 0.21 & 35.80 ± 0.77 & 67.41 ± 2.83 & 65.65 ± 0.36 & 82.29 ± 1.87 & 78.80 ± 0.18 & 5.0 \\
EHNN-Transformer & 93.27 ± 6.59 & \cellcolor{gray!60}\textbf{81.42 ± 0.53} & \cellcolor{gray!60}\textbf{100.00 ± 0.00} & \cellcolor{gray!60}\textbf{89.60 ± 1.36} & \cellcolor{gray!60}\textbf{98.28 ± 0.18} & \cellcolor{blue!30}36.48 ± 0.40 & \cellcolor{gray!60}\textbf{71.53 ± 2.59} & \cellcolor{gray!60}\textbf{68.73 ± 0.35} & \cellcolor{gray!60}\textbf{85.09 ± 2.05} & \cellcolor{gray!60}\textbf{80.05 ± 0.27} & \cellcolor{gray!60}\textbf{1.6} \\
\bottomrule
\end{tabular}
}
\end{center}
\end{table}

We use 8 datasets used in Chien~et.~al.~\cite{chien2022you} (Table~\ref{table:node_classification}).
Among them, three datasets (20Newsgroups, Mushroom, Zoo) are from the UCI Categorical Machine Learning Repository, and two (ModelNet40, NTU2012) are from computer vision domain where the objective is to classify visual objects.
Other three (Yelp, House, Walmart) were crafted in Chien~et.~al.~\cite{chien2022you}.
In Yelp, the nodes correspond to restaurants where the node labels correspond to the number of stars provided in the yelp "restaurant" catalog.
For House, each node and label is a member of the US House of Representatives and their political party.
Nodes of Walmart represent products where node label corresponds to product category, and set of products purchased together are tied with a hyperedge.
Since House and Walmart does not contain node features, the node features are created by adding Gaussian noise to one-hot node labels.
We fix node feature dimension to 100 as in prior work.

We use a composition of two layers that maps node features to hyperedge features ($V\to E$) and maps the features back to nodes ($E\to V$) as a module, and use a single module for EHNN-MLP and use a stack of two modules for EHNN-Transformer.
For hyperparameter search, we fix learning rate as 0.001 and run grid search over hidden dimension \{64, 128, 256, 512\}, weight decays \{0, 0.00001\}, and MLP dropout \{0, 0.2\}.
For EHNN-Transformer, we also search number of heads over \{4, 8\} for multi-head attention, attention output dropout \{0, 0.1\}, and classifier hidden dimension \{64, 128\}.
The final selection of hyperparameters based on grid search are outlined in Table~\ref{table:node_classification_hyperparameter}.

Along with the average accuracy reported in the main text, we report the standard deviation over 20 runs with random train/val/test splits and model initialization in Table~\ref{table:semi_supervised_with_std}.

\subsubsection{Visual Keypoint Matching}
For evaluation under inductive setting, we test the performance of EHNN on keypoint matching benchmarks implemented in the ThinkMatch repository~\cite{wang2019neural}.
We borrow two real image datasets, Willow ObjectClass~\cite{cho2013learning} and PASCAL-VOC~\cite{everingham2010the, bourdev2009poselets} with Berkeley annotations (Table~\ref{table:keypoint_matching}), as well as provided train/test splitting pipelines from the repository.
For training, we use binary cross entropy loss between two permutation matrices: one from ground-truth matching and another from node classification on the association hypergraph.
Performance at test time is evaluated by measuring the matching accuracy via F1-score.
We compare our EHNN methods against 10 different methods, by reproducing their performance based on implementation in ThinkMatch\footnote{\url{https://github.com/Thinklab-SJTU/ThinkMatch}}.
Several methods are excluded due to numerical instability errors.
For EHNN methods, we follow the same keypoint feature extraction and association hypergraph construction procedure as in NHGM-v2: we simply replace the local message-passing GNN module in NHGM-v2 with an EHNN-MLP/Transformer.

To find optimal hyperparameters for EHNN-MLP, we run grid search over hidden dimension sizes \{16, 32, 64, 96, 128\}. For the Transformer variant, we search through number of layers \{1, 2, 3\} and number of attention heads \{2, 4, 8, 16\} in addition to hidden dimension sizes.
For EHNN-Transformer we also check applying dropouts within $[0.1, 0.5]$ separately to global and local attention outputs respectively.
The final selected hyperparameters are outlined in Table~\ref{table:keypoint_matching_hyperparameter}.
For training EHNN methods, we use the default setting used for NHGM-v2 in ThinkMatch.
For Willow, we train for 10 epochs with learning rate that starts at 2e$^{-4}$ and decays into half at epoch 2.
For PASCAL-VOC, we train for 20 epochs with learning rate that also starts 2e$^{-4}$ and decays into half at epoch 2.